\documentclass{article}

\usepackage{microtype}
\usepackage{graphicx}
\usepackage{subcaption}
\usepackage{enumitem}
\usepackage{booktabs}
\usepackage[margin=1in]{geometry}
\usepackage{placeins}
\usepackage{subcaption}
\usepackage[T1]{fontenc}

\usepackage{mathtools, amsthm, amsfonts, amssymb}
\usepackage{caption}
\usepackage{algorithm,algorithmic}

\DeclareMathOperator{\sech}{sech}
\DeclareMathOperator*{\argmax}{arg\,max}
\DeclareMathOperator*{\argmin}{arg\,min}
\DeclarePairedDelimiter{\norm}{\lVert}{\rVert}
\DeclarePairedDelimiter{\abs}{\lvert}{\rvert}
\DeclareMathOperator{\E}{\mathbb{E}}
\DeclareMathOperator{\Tr}{Tr}
\newtheorem{theorem}{Theorem}[section]
\newtheorem{lemma}[theorem]{Lemma}
\newtheorem{prop}[theorem]{Proposition}
\newtheorem{corollary}{Corollary}[section]
\newcommand{\MSE}{\text{MSE}}
\newcommand{\eps}{\epsilon}
\newcommand{\toas}{\overset{\text{a.s.}}{\to}}

\newcommand{\R}{\mathbb{R}}

\usepackage{xcolor}
\definecolor{darkred}{RGB}{100,0,0}
\definecolor{darkgreen}{RGB}{0,100,0}
\definecolor{darkblue}{RGB}{0,0,150}

\usepackage[colorlinks=true,
            linkcolor=darkred,
            citecolor=darkgreen,
            urlcolor=darkblue]{hyperref}

\title{Active embedding search via noisy paired comparisons}

\makeatletter
\def\blfootnote{\xdef\@thefnmark{}\@footnotetext}
\makeatother

\begin{document}

\author{Gregory H.~Canal,
    Andrew K.~Massimino,
    Mark A.~Davenport,
    and Christopher J.~Rozell}

\maketitle

\begin{abstract}
Suppose that we wish to estimate a user's preference vector $w$ from paired comparisons of the form ``does user $w$ prefer item $p$ or item $q$?,'' where both the user and items are embedded in a low-dimensional Euclidean space with distances that reflect user and item similarities. Such observations arise in numerous settings, including psychometrics and psychology experiments, search tasks, advertising, and recommender systems. In such tasks, queries can be extremely costly and subject to varying levels of response noise; thus, we aim to \emph{actively} choose pairs that are most informative given the results of previous comparisons. We provide new theoretical insights into the benefits and challenges of greedy information maximization in this setting, and develop two novel strategies that maximize lower bounds on information gain and are simpler to analyze and compute respectively. We use simulated responses from a real-world dataset to validate our strategies through their similar performance to greedy information maximization, and their superior preference estimation over state-of-the-art selection methods as well as random queries.
\blfootnote{School of Electrical and Computer Engineering, Georgia Institute of Technology, Atlanta, GA, 30332 USA
    \\(e-mail: \{%
        \href{mailto:gregory.canal@gatech.edu}{gregory.canal},%
        \href{mailto:massimino@gatech.edu}{massimino},%
        \href{mailto:mdav@gatech.edu}{mdav},%
        \href{mailto:crozell@gatech.edu}{crozell}%
    \}@gatech.edu)}%
\end{abstract}

\section{Introduction}
We consider the task of user preference learning, where we have a set of \emph{items} (e.g.,\ movies, music, or food)
embedded in a Euclidean space and aim to represent the preferences of a \emph{user} as a continuous point in the same space (rather than simply a rank ordering over the items) so that their preference point is close to items the user likes and far from items the user dislikes. To estimate this point, we consider a system using the
\emph{method of paired comparisons}, where during a sequence of interactions a user chooses which of two given items they prefer~\cite{david1963method}. For instance, to characterize a person's taste in food, we might ask them which one of two dishes they would rather eat for a number of different pairs of dishes. The recovered preference point can be
used in various tasks, for instance
in the recommendation of nearby items, personalized product creation, or
clustering of users with similar preferences. We refer to the entire process of querying via paired comparisons and continuous preference point estimation as \emph{pairwise search}, and note that this is distinct from the problem of searching for a single discrete item in the fixed dataset. A key goal of ours is to \emph{actively} choose the items in each query and demonstrate the advantage over non-adaptive selection.

More specifically, given $N$ items, there are $O(N^2)$ possible paired comparisons. Querying all such pairs is not only prohibitively expensive for large datasets, but also unnecessary since not all queries are informative; some queries are rendered obvious by the accumulation of evidence about the user's preference point, while others are considered ambiguous due to noise in the comparison process. \emph{Given these considerations, the main contribution of this work is the design and analysis of two new query selection algorithms for pairwise search that select the most informative pairs by directly modeling redundancy and noise in user responses}. While previous active algorithms have been designed for related paired comparison models, none directly account for probabilistic user behavior as we do here. To the best of our knowledge our work is the first attempt to search a low-dimensional embedding for a \emph{continuous} point via paired comparisons while directly modeling \emph{noisy} responses.

Our approach builds upon the popular technique in active learning and Bayesian experimental design of greedily maximizing information gain
\cite{settles2012active,lindley1956measure,mackay1992information}. In our setting, this corresponds to selecting pairs that maximize the mutual information between the user's response and the unknown location of their preference point. We provide new theoretical and computational insights into relationships between information gain maximization and estimation error minimization in pairwise search, and present a lower bound on the estimation error achievable by any query strategy.

Due to the known difficulty of analyzing greedy information gain maximization \cite{chen2015sequential} and the high computational cost of estimating mutual information for each pair in a pool, we propose two strategies that each maximize new lower bounds on information gain and are simpler to analyze and compute respectively. We present upper and lower bounds on the performance of our first strategy, which then motivates the use of our second, computationally cheaper strategy. We then demonstrate through simulations using a real-world dataset how both strategies perform comparably to
information maximization while outperforming state-of-the-art techniques and randomly selected queries.

\section{Background}

\subsection{Observation model}

Our goal in this paper is to estimate a user's preference point (denoted as vector $w$) with respect to a given low-dimensional embedding of items constructed such that distances between items are consistent with item similarities, where similar items are close together and dissimilar items are far apart. While many items (e.g., images) exist in their raw form in a high-dimensional space (e.g., pixel space), this low-dimensional representation of items and user preferences offers the advantage of simple Euclidean relationships that directly capture notions of preference and similarity, as well as mitigating the effects of the curse of dimensionality in estimating preferences. Specifically, we suppose user preferences can be captured via an \emph{ideal point model}
in which each item and user is represented using a common set of parameters in $\mathbb{R}^d$, and that a user's overall
preference for a particular item decreases with the
distance between that item and the user's ideal point $w$ \cite{coombs1950psychological}.
This means that any item placed exactly at the user
would be considered ``ideal'' and would be the most
preferred over all other items. Although this model can be applied to the situation where a particular
item is sought, in general we do \emph{not} assume the user point $w$ to be co-located with any item.

The embedding of the items can be constructed through a training set of triplet comparisons (paired comparisons regarding similarity of two items to a third reference item) using one of several standard non-metric embedding techniques such as the Crowd Kernel Learning~\cite{tamuz_adaptively_2011} or
Stochastic Triplet Embedding methods~\cite{van2012stochastic}. In this study, we assume that such an embedding is given, presumably acquired through a large set of crowdsourced training triplet comparisons. We do not consider this training set to be part of the learning cost in measuring a pairwise search algorithm's efficiency, since our focus here is on efficiently choosing paired comparisons to search an \emph{existing} embedding.

In this work, we assume a noisy ideal point model where the probability of a user located at $w$ choosing
item $p$ over item $q$ in a paired comparison is modeled using
\begin{equation}
	P(p \prec q) = f(k_{pq}(\norm{w - q}^2
    	- \norm{w - p}^2)),\label{eq:response}
\end{equation}
where $p\prec q$ denotes ``item $p$ is preferred to item $q$,''
$f(x) = 1/(1+e^{-x})$ is the logistic function, and $k_{pq}\in[0,\infty)$ is the pair's \emph{noise constant}, which represents roughly
the signal-to-noise ratio of
a particular query and may depend on the
values of $p$ and $q$. This type of
logistic noise model is common in psychometrics literature
and bears similarity to the Bradley--Terry
model~\cite{bradley1952rank}.

Note that \eqref{eq:response} can also be written as
\[
P(p \prec q) = f(k_{pq}(a^Tw - b)),
\]
where $a = 2(p-q)$ and $b = \norm{p}^2 - \norm{q}^2$ encode the normal vector and threshold of a hyperplane bisecting items $p$ and $q$. After a number of such queries, the response model in \eqref{eq:response} for each query can be multiplied to form a posterior belief about the location of $w$, as depicted in
Figure~\ref{fig:posterior}.

Note that we allow the noise constant $k_{pq}$ to differ for each item pair to allow for differing
user behavior depending on the geometry of the items being compared.
When $k_{pq}\to\infty$, this supposes a user's selection is
made with complete certainty and cannot be erroneous.
Conversely, $k_{pq}=0$ corresponds to choosing items randomly with
probability $1/2$. Varying $k_{pq}$ allows for differing reliability when items are far apart versus when they are close together.

Some concrete examples for setting this parameter are:
\begin{align}
   \text{constant}: k_{pq}^{(1)} &= k_0,
        \tag{K1}\label{eq:k1}
   \\ \text{normalized}: k_{pq}^{(2)} &= k_0\norm{a}^{-1} = \frac{1}{2} k_0\norm{(p-q)}^{-1},
        \tag{K2}\label{eq:k2}
   \\ \text{decaying}: k_{pq}^{(3)} &= k_0\exp(-\norm{a}) \notag
   \\ &= k_0\exp(-2\norm{(p-q)}).
        \tag{K3}\label{eq:k3}
\end{align}

\begin{figure}[htbp]\centering
\includegraphics[height=1.4in]{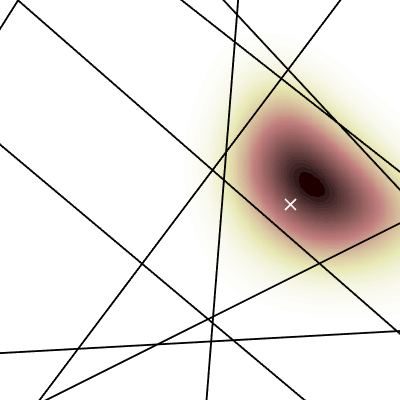}
\caption{Paired comparisons between items
can be thought of as a set of noisy hyperplane queries.
In the high-fidelity case, this
uniquely identifies a convex region of $\R^d$.
In general, we have a posterior distribution
which only approximates the shape of the ideal cell around the true user point, depicted with an x.}
\label{fig:posterior}
\end{figure}

\subsection{Related work}
There is a rich literature investigating statistical inference from paired comparisons and related ordinal query types. However, many of these works target a different problem than considered here, such as constructing item embeddings \cite{tamuz_adaptively_2011}, training classifiers \cite{guo2018experimental}, selecting arms of bandits \cite{jamieson2015sparse}, and learning rankings
\cite{wauthier2013efficient,cao2007learning,chen2015spectral,eriksson2013learning,shah2017simple} or scores \cite{shah2016estimation,negahban2012iterative} over items.

Paired comparisons have also been used for learning user preferences: \cite{qian2015learning} models user preferences as a vector, but preferences are modeled as linear weightings of item features rather than by relative distances between the user and items in an embedding, resulting in a significantly different model (e.g., monotonic) of preference. \cite{brochu2008active} considers the task of actively estimating the maximizer of an unknown preference function over items, while \cite{houlsby2012collaborative} and \cite{chu2005preference} actively approximate the preference function itself, the former study notably using information gain as a metric for selecting queries. Yet, these approaches are not directly comparable to our methods since they do not consider a setting where user points are assigned within an existing low-dimensional item embedding. \cite{tamuz_adaptively_2011} does consider the same item embedding structure as our setting and actively chooses paired comparisons that maximize information gain for search, but only seeks \emph{discrete items} within a fixed dataset rather than estimating a \emph{continuous} preference vector as we do here. Furthermore we provide novel insights into selecting pairs via information gain maximization, and mainly treat information gain for pairwise search as a baseline in this work since our primary focus is instead on the development, analysis, and evaluation of alternative strategies inspired by this approach.

The most directly relevant prior work to our setting consists of the theory and algorithms developed in \cite{massimino2018you} and \cite{jamieson2011active}. In \cite{massimino2018you}, item pairs are selected in stages to approximate a Gaussian point cloud that surrounds the current user point estimate and dyadically shrinks in size with each new stage. In \cite{jamieson2011active}, previous query responses define a convex polytope in $d$ dimensions (as in Figure~\ref{fig:posterior}), and their algorithm only selects queries whose bisecting hyperplanes intersect this feasible region. While this algorithm in its original form only produces a rank ordering over the embedding items, for the sake of a baseline comparison we extend it to produce a preference point estimate from the feasible region. Neither of these studies fundamentally models or handles noise in their active selection algorithms; slack variables are used in the user point estimation of \cite{massimino2018you} to allow for contradicting query responses, but the presence of noise is not considered when selecting queries. In an attempt to filter non-persistent noise (the type encountered in our work), \cite{jamieson2011active} simply repeat each query multiple times and take a majority vote as the user response, but the items in the query pair are still selected using the same method as in the noiseless setting. Nevertheless, these methods provide an important baseline.

\section{Query selection}
We now proceed to describe the pair selection problem in detail along with various theoretical and computational considerations. We show that the goal of selecting pairwise queries to minimize estimation error leads naturally to the strategy of information maximization and subsequently to the development of our two novel selection strategies.

\subsection{Minimizing estimation error}

Let $W \in \mathbb{R}^d~(d\ge 2)$ denote a random vector encoding the user's preference point, assumed for the sake of analysis to be drawn from a uniform distribution over the hypercube $[-\frac{1}{2},\frac{1}{2}]^d$ denoted by the prior density of $p_0(w)$. Unless noted otherwise, we denote random variables with uppercase letters, and specific realizations with lowercase letters. Let $Y_i \in \{0,1\}$ denote the binary response to the $i^{\text{th}}$ paired comparison involving items $p_i$ and $q_i$, with $Y_i=0$ indicating a preference for $p_i$ and $Y_i = 1$ a preference for $q_i$. After $i$ queries, we have the vector of responses $Y^i = \{Y_1,Y_2,\dots Y_i\}$. We assume that each response $Y_i$ is conditionally independent from previous responses $Y^{i-1}$ when conditioned on preference $W$. Applying this assumption in conjunction with a recursive application of Bayes' rule, after $i$ queries we have a posterior density of
\begin{equation}
    p_i(w) \equiv p(w|Y^i) = \frac{p_0(w)\prod_{j=1}^i p(Y_j|w)}{p(Y_i|Y^{i-1})}
    \label{eq:posterior_density}
\end{equation}
where $p(Y_i|w)$ is given by the model in \eqref{eq:response}. This logistic likelihood belongs to the class of \emph{log-concave} (LCC) distributions, whose probability density functions $f(w)$ satisfy $f(\alpha w_1 + (1-\alpha) w_2) \geq f(w_1)^\alpha f(w_2)^{1-\alpha}$ for any $w_1,w_2 \in \mathbb{R}^d$ and $0 \leq \alpha \leq 1$. Since $p_0(w)$ is log-concave and products of log-concave functions are also log-concave \cite{saumard2014log}, we have that the posterior density given in \eqref{eq:posterior_density} is log-concave.

Suppose that after $i$ queries, the posterior $p_i(w)$ is used to produce a Bayesian user point estimate $\widehat{W}_i$. We denote the mean squared error for this estimate by $\MSE_i = \E_{W|Y^i}[\norm{W - \widehat{W}_i}_2^2]$, which provides a direct measure of our estimation error and is a quantity we wish to minimize by adaptively selecting queries based on previous responses. One approach might be to greedily select an item pair such that $\MSE_{i+1}$ is minimized in expectation after the user responds. However, this would require both updating the posterior distribution and estimating $\MSE_{i+1}$ for each possible response over all item pairs. This would be very computationally expensive since under our model there is no closed-form solution for $\MSE_{i+1}$, and so each such evaluation requires a ``lookahead'' batch of Monte Carlo samples from the posterior. Specifically, if $S$ posterior samples are generated for each $\MSE_{i+1}$ evaluation over a candidate pool of $M$ pairs at a computational cost of $C$ per sample generation, and $\MSE_{i+1}$ is estimated with $O(dS)$ operations per pair, this strategy requires $O((C + d)SM)$ computations to select each query. This is undesirable for adaptive querying settings where typically data sets are large (resulting in a large number of candidate pairwise queries) and queries need to be selected in or close to real-time.

Instead, consider the covariance matrix of the user point posterior after $i$ queries, denoted as \[\Sigma_{W|Y^i} = \E[(W-\E[W|Y^i])(W-\E[W|Y^i])^T|Y^i].\] For the minimum mean squared error (MMSE) estimator, given by the posterior mean $\widehat{W}_i = \E[W|Y^i]$, we have
\[
     \MSE_i = \Tr(\Sigma_{W|Y^i}) \geq d \abs{\Sigma_{W|Y^i}}^\frac{1}{d}
\]
where the last inequality follows from the arithmetic-geometric mean inequality (AM--GM) \cite{boyd2004convex}. This implies that a necessary condition for a low $\MSE$ is for the \emph{posterior volume}, defined here as the determinant of the posterior covariance matrix, to also be low. Unfortunately, actively selecting queries that greedily minimize posterior volume is too computationally expensive to be useful in practice since this also requires a set of ``lookahead'' posterior samples for each candidate pair and possible response, resulting in a computational complexity of $O(((C + d^2)S + d^3)M)$ to select each query from the combined cost per pair of generating samples ($O(CS)$), estimating $\Sigma_{W|Y^i}$ ($O(d^2S)$), and calculating $\abs{\Sigma_{W|Y^i}}$ ($O(d^3)$).

\subsection{Information theoretic framework}

By utilizing statistical tools from information theory, we can select queries that approximately minimize posterior volume (and hence tend to encourage low $\MSE$) at a more reasonable computationally cost. Furthermore, an information theoretic approach provides convenient analytical tools which we use to provide performance guarantees for the query selection methods we present.

Towards this end, we define the \emph{posterior entropy} as the differential entropy of the posterior distribution after $i$ queries:
\begin{equation}
    h_i(W)\equiv h(W|y^i) = -\int_w p_i(w) \log_2 (p_i(w)) dw.
\end{equation}
As we show in the following lemma, the posterior entropy of LCC distributions is both upper and lower bounded by a monotonically increasing function of posterior volume, implying that low posterior entropy is both necessary and sufficient for low posterior volume, and hence a necessary condition for low $\MSE$. The proofs of this lemma and subsequent results are provided in the supplementary material.

\begin{lemma}\label{lem:LCC_ent_Lower}
For a LCC posterior distribution $p(w|Y^i)$ in $d\ge 2$ dimensions,
where $c_d = e^2 d^2/(4\sqrt{2} (d+2))$,
\[
\frac{d}{2} \log_2\frac{2 \abs{\Sigma_{W|Y^i}}^\frac{1}{d}}{e^2 c_d} \leq h_i(W) \leq \frac{d}{2} \log_2(2\pi e \abs{\Sigma_{W|Y^i}}^\frac1d).
\]
\end{lemma}
This relationship between $\MSE$, posterior volume, and posterior entropy suggests a strategy of selecting queries that minimize the posterior entropy after each query. Since the actual user response is unknown at the time of query selection, we seek to minimize the \emph{expected} posterior entropy after a response is made, i.e., $\E_{Y_{i+1}}[h_{i+1}(W)|y^i]$. Using a standard result from information theory, we have $\E_{Y_{i}}[h_{i}(W)|y^{i-1}] = h_i(W) - I(W;Y_{i}|y^{i-1})$, where $I(W;Y_{i}|y^{i-1})$ is the \emph{mutual information} between the location of the unknown user point and the user response, conditioned on previous responses \cite{cover2012elements}. Examining this identity, we observe that selecting queries that minimize the expected posterior entropy is equivalent to selecting queries that maximize the mutual information between the user point and the user response, referred to here as the \emph{information gain}.

In this setting, it is generally difficult to obtain sharp performance bounds for query selection via information gain maximization. Instead, we use information theoretic tools along with Lemma \ref{lem:LCC_ent_Lower} to provide a lower bound on $\MSE$ for \emph{any} estimator and query selection scheme in a manner similar to \cite{prasad2010certain} and \cite{cover2012elements}:

\begin{theorem}
For any user point estimate given by $\widehat{W}_i$ after $i$ queries, the $\MSE$ (averaged over user points and query responses) for any selection strategy is bounded by
\[\E_{W,Y^i}\norm{W - \widehat{W}_i}_2^2 \geq \frac{d2^{-2\frac{i}{d}}}{2\pi e}.\] \label{thm:lower1}
\end{theorem}
This result implies that the best rate of decrease in $\MSE$ one can hope for is exponential in the number of queries and slows down in a matter inversely proportional to the dimension, indicating quicker possible preference convergence in settings with lower dimensional embeddings. To estimate the information gain of a query, we can use the symmetry of mutual information to write
\begin{align}
    I(W;Y_{i}|y^{i-1}) &= H(Y_{i}|y^{i-1}) - H(Y_{i}|W,y^{i-1}) \label{eq:iteration_information}\\
    H(Y_{i}|y^{i-1}) &= -\sum_{\mathclap{Y_{i} \in \{0,1\}}} p(Y_{i}|y^{i-1}) \log_2 p(Y_{i}|y^{i-1}) \label{eq:response_entropy}\\
    H(Y_{i}|w,y^{i-1}) &= -\sum_{\mathclap{Y_{i} \in \{0,1\}}} p(Y_{i}|w) \log_2 p(Y_{i}|w)\\
    H(Y_{i}|W,y^{i-1}) &= \E_{W|y^{i-1}}[H(Y_{i}|W,y^{i-1})]. \label{eq:cond_ent}
\end{align}
Unlike the greedy $\MSE$ and posterior volume minimization strategies, information gain estimation only requires a \emph{single} batch of posterior samples at each round of query selection, which is used to estimate the discrete entropy quantities in \eqref{eq:iteration_information}--\eqref{eq:cond_ent}. \eqref{eq:iteration_information} can be estimated in $O(dS)$ operations per pair, resulting in a computational cost of $O(dSM)$ for selecting each query, which although more computationally feasible than the methods proposed so far is still likely prohibitive for highly accurate information gain estimates over a large pool of candidate pairs.

Because of these analytical and computational challenges, we develop two strategies that mimic the action of maximizing information gain while being more analytically and computationally tractable, respectively. In the next section we present our first strategy, which we analyze for more refined upper and lower bounds on the number of queries needed to shrink the posterior to a desired volume. Then we introduce a second strategy which benefits from reduced computational complexity while still remaining theoretically coupled to maximizing information gain.

\subsection{Strategy 1: Equiprobable, Max-variance}
\label{sec:EPMV}
In developing an approximation for information gain maximization, consider the scenario where \emph{arbitrary} pairs of items can be generated (unconstrained to a given dataset), resulting in a bisecting hyperplane parameterized by $(a_i,b_i)$.
In practice, such queries might correspond to the generation of synthetic items via tools such as generative adversarial networks \cite{goodfellow2014generative}. With this freedom, we could consider an \emph{equiprobable} query strategy where $b_i$ is selected so that each item in the query will be chosen by the user with probability $\frac12$. This strategy is motivated by the fact that the information gain of query $i$ is upper bounded by $H(Y_{i}|y^{i-1})$, which is maximized
if and only if the response probability is equiprobable~\cite{cover2012elements}.

To motivate the selection of query hyperplane directions, we define a query's \emph{projected variance}, denoted as $\sigma_i^2$, as the variance of the posterior marginal in the direction of a query's hyperplane, i.e.,  $\sigma_i^2 = a_i^T\Sigma_{W|y^{i-1}} a_i$. This corresponds to a measure of how far away the user point is from the hyperplane query, in expectation over the posterior distribution. With this notation, we have the following lower bound on information gain for equiprobable queries.
\begin{prop}\label{prop:milb}
For any ``equiprobable'' query scheme with noise constant $k_i$ and projected variance $\sigma_i^2$, for any choice of constant $0 \leq c \leq 1$ we have
\[
I(W;Y_i|y^{i-1}) \geq \biggl(1-h_b\Bigl(f\Bigl(\frac{c k_i \sigma_i}{2}\Bigr)\Bigr)\biggr) (1- c)
\eqqcolon L_{c,k_i}(\sigma_i)
\]
where $h_b(p) = -p\log_2 p - (1-p)\log_2(1-p)$.
\end{prop}
This lower bound is monotonically increasing with $k_i \sigma_i$ and achieves a maximum information gain of 1 bit at $k_i \to \infty$ and/or $\sigma_i \to \infty$ (with an appropriate choice of $c$). This suggests choosing $a_i$ that \emph{maximize projected variance} in addition to selecting $b_i$ according to the equiprobable strategy. Together, we refer to the selection of equiprobable queries in the direction of largest projected variance as the equiprobable-max-variance scheme, or EPMV for short.

Our primary result concerns the expected number of comparisons
(or query complexity) sufficient to reduce the posterior volume below a specified threshold set a priori,
using EPMV.

\begin{theorem}\label{thm:guarantee}
For the EPMV query scheme with each selected query satisfying $k_i\norm{a_i} \!\geq\! k_\text{min}$ for some constant $k_\text{min}\!>\!0$, consider the stopping time $T_\varepsilon\!=\!\min \{i:\abs{\Sigma_{W|y^i}}^\frac{1}{d}\!< \!\varepsilon\}$ for stopping threshold $\varepsilon > 0$. For $\tau_1 = \frac{d}{2} \log_2(\frac{1}{2\pi e \varepsilon})$ and $\tau_2 = \frac{d}{2} \log_2\frac{e^2 c_d}{2 \varepsilon}$, we have
\[
\tau_1 \leq E[T_\varepsilon] \leq
\tau_2 + \frac{\tau_2+1}{l(\tau_2)} - \frac{1}{l(\tau_2)} \int_{0}^{\tau_2} l(x) dx
\]
where $l(x) =L_{c,k_\text{min}}\Bigl(\frac{2^{\frac{-x}{d}}}{\sqrt{2\pi e}}\Bigr)$ for any constant $0 \leq c \leq 1$ as defined in Proposition~\ref{prop:milb}.
Furthermore, the lower bound is true for any query selection scheme.
\end{theorem}
This result follows from a martingale
stopping-time analysis of the entropy
at each query. Our next theorem presents a looser upper bound, but is more
easily interpretable.

\begin{theorem}\label{thm:stoppingorder}
The EPMV scheme, under the same
assumptions as in Theorem~\ref{thm:guarantee},
satisfies
\[
\E[T_\varepsilon] = O\left(d\log\frac{1}{\varepsilon} + \left(\frac{1}{\varepsilon k_\text{min}^2}\right) d^2\log\frac{1}{\varepsilon}\right).
\]
Furthermore, for any query scheme,
$
\E[T_\varepsilon] = \Omega\left(d \log \frac{1}{\varepsilon}\right).
$
\end{theorem}

This result has a favorable dependence on the dimension $d$,
and the upper bound can be interpreted as a blend between two rates, one of which matches that of the generic lower bound.
The second term in the upper bound provides some evidence that our ability to recover
$w$ worsens
as $k_\text{min}$ decreases.
This is intuitively unsurprising
since small $k_\text{min}$ corresponds to the case where queries
are very noisy. We hypothesize that the absence of such a penalty term in the lower bound is an artifact of our analysis, since increasing noise levels (i.e., decreasing $k_\text{min}$) should limit achievable performance by any querying strategy.
On the other hand, for asymptotically
large $k_i$, we have the following corollary:

\begin{corollary}\label{cor:optimality}
In the noiseless setting ($k_\text{min} \to \infty$), EPMV has optimal expected stopping time complexity for posterior volume stopping.
\end{corollary}
\begin{proof}
When $k_\text{min} \to \infty$, from Theorem \ref{thm:stoppingorder} $\E[T_\varepsilon] = O\left(d\log\frac{1}{\varepsilon}\right)$; for any scheme, $\E[T_\varepsilon] = \Omega\left(d \log \frac{1}{\varepsilon}\right)$.
\end{proof}
Taken together, these results suggest that EPMV is optimal with respect to posterior volume minimization up to a penalty term which decreases to zero for large noise constants. While low posterior volume is only a necessary condition for low $\MSE$, this result could be strengthened to an upper bound on $\MSE$ by bounding the condition number of the posterior covariance matrix, which is left to future work. Yet, as we empirically demonstrate in Section \ref{sec:sim}, in practice our methods are very successful in reducing $\MSE$.

While EPMV was derived under the assumption of arbitrary hyperplane queries, depending on the application we may have to select a pair from a fixed pool of items in a given dataset. For this purpose we propose a metric for any candidate pair that, when maximized over all pairs in a pool, approximates the behavior of EPMV. For a pair with items $p$ and $q$ in item pool $\mathcal{P}$, let $a_{pq} = 2(p-q)$ and $b_{pq}=\norm{p}^2 - \norm{q}^2$ denote the weights and threshold parameterizing the bisecting hyperplane. We choose a pair that maximizes the utility function (for some $\lambda > 0$)
\begin{gather}
    \eta_1(p,q;\lambda) = k_{pq}\sqrt{a_{pq}^T
    \Sigma_{W|Y^{i-1}} a_{pq}}-\lambda
    \Bigl|\widehat{p}_1 - \frac12 \Bigr|  \label{eq:equimv}
    \\
    \widehat{p}_1 = P(Y_i\!=\!1|Y^{i-1})= \E_{W|Y^{i-1}}[f(k_{pq}(a_{pq}^TW - b_{pq}))].  \notag
\end{gather}
This has the effect of selecting queries which are close to equiprobable
and align with the direction of largest variance, weighted by $k_{pq}$ to prefer higher fidelity queries. While $\Sigma_{W|Y^{i-1}}$ can be estimated once from a batch of posterior samples, $\widehat{p}_1$ must be estimated for each candidate pair in $O(dS)$ operations, resulting in a computational cost of $O(dSM)$ which is on the same order as directly maximizing information gain. For this reason, we develop a second strategy that approximates EPMV while significantly reducing the computational cost.

\subsection{Strategy 2: Mean-cut, Max-variance}
\label{sec:MCMV}
Our second strategy is a \emph{mean-cut} strategy where $b_i$ is selected such that the query hyperplane passes through the posterior mean, i.e.\ $a_i^T \E[W|Y^{i-1}] - b_i = 0$. For such a strategy, we have the following proposition:

\begin{prop}\label{prop:meancut_MI_lower}
For mean-cut queries with noise constant $k_i$ and projected variance $\sigma_i^2$ we have
\[
\Big\vert p(Y_{i}|y^{i-1}) - \frac{1}{2}\Big\vert \leq \frac{e-2}{2e} +\frac{\ln 2}{k_i
    \sigma_i}
\]
\[\text{and,}\,\,
I(W;Y_i|y^{i-1}) \geq h_b\Biggl(\frac{1}{e}-\frac{\ln 2}{k_i
    \sigma_i}
\Biggr) -\frac{\pi^2 (\log_2 e)}{3k_i
    \sigma_i
    }.
\]
\label{prop:meancut_infobounds}
\end{prop}
For large projected variances, we observe that $\abs{p(Y_{i}|y^{i-1}) - \frac{1}{2}} \lessapprox 0.14$, suggesting that mean-cut queries are somewhat of an approximation to equiprobable queries in this setting. Furthermore, notice that the lower bound to information gain in Proposition \ref{prop:meancut_infobounds} is a monotonically increasing function of the projected variance. As $\sigma_i \to \infty$, this bound approaches $h_b(1/e) \approx 0.95$ which is nearly sharp since a query's information gain is upper bounded by 1 bit. This implies some correspondence between maximizing a query's information gain and maximizing the projected variance, as was the case in EPMV. Hence, our second strategy selects \emph{mean-cut, maximum variance} queries (referred to as MCMV) and serves as an approximation to EPMV while still maximizing a lower bound on information gain.

\begin{algorithm}[tbp]
\caption{Pairwise search with noisy comparisons}
\label{algo:utility}
\begin{algorithmic}
	\REQUIRE item set $\mathcal{X}$, parameters $S$, $\beta$, $\lambda$
	\STATE $\mathcal{P} \leftarrow$ set of all pairwise queries from items in $\mathcal{X}$
    \STATE $\widetilde{W_0}, \mu_0, \Sigma_0 \leftarrow$ initialize from
    samples of prior
    \FOR{$i = 1$ \textbf{to} $T$}
        \STATE $\mathcal{P_\beta} \leftarrow$ uniformly downsample $\mathcal{P}$ at rate $0 < \beta \leq 1$
        \STATE \textbf{InfoGain:}~$ \displaystyle p_i,q_i \leftarrow \argmax_{p,q \in \mathcal{P}_\beta} \, \eta_0(p,q; \widetilde{W}_{i-1})$
        \STATE \textbf{EPMV:}~$ \displaystyle p_i,q_i \leftarrow \mathop{\argmax}_{p,q \in \mathcal{P}_\beta} \,
                \eta_1(p,q;\lambda, \widetilde{W}_{i-1})$
        \STATE \textbf{MCMV:}~$ \displaystyle p_i,q_i \leftarrow \argmax_{p,q \in \mathcal{P}_\beta} \,
                \eta_2(p,q; \lambda, \mu_{i-1}, \Sigma_{i-1})$
        \STATE $y_i \leftarrow $ PairedComparison$(p_i,q_i)$
        , $y^i \leftarrow y_i\cup y^{i-1}$.
        \STATE $\widetilde{W_i} \leftarrow$ batch of $S$ samples from
            posterior $W|Y^{i}$
        \STATE $\mu_i, \Sigma_i \leftarrow \operatorname{Mean}(\widetilde{W}_i),
        \operatorname{Covariance}(\widetilde{W}_i)$
        \STATE $\widehat{W}_i \leftarrow \mu_i$
    \ENDFOR
    \ENSURE user point estimate $\widehat{W}_T$
\end{algorithmic}
\end{algorithm}

For implementing MCMV over a fixed pool of pairs (rather than arbitrary hyperplanes), we calculate the orthogonal distance of each pair's hyperplane to the posterior mean as $\abs{a_{pq}^T\E[W|Y^{i-1}] - b_{pq}}/\norm{a_{pq}}_2$ and the projected variance as $a_{pq}^T \Sigma_{W|Y^{i-1}} a_{pq}$. We choose a pair that maximizes the following function which is a tradeoff (tuned by $\lambda > 0$) between minimizing distance to the posterior mean, maximizing noise constant, and maximizing projected variance:
\begin{equation}
    \eta_2(p,q;\lambda) = k_{pq}\sqrt{a_{pq}^T \Sigma_{W|Y^{i-1}} a_{pq}}
      - \lambda \frac{\abs{a_{pq}^T\E[W|Y^{i-1}] - b_{pq}}}{\norm{a_{pq}}_2}. \label{eq:MCMV_tradeoff}
\end{equation}
This strategy is attractive from a computational standpoint since the posterior mean $\E[W|Y^{i-1}]$ and covariance $\Sigma_{W|Y^{i-1}}$ can be estimated \emph{once} in $O(d^2 S)$ computations, and subsequent calculation of the hyperplane distance from mean and projected variance requires only $O(d^2)$ computations per pair. Overall, this implementation of the MCMV strategy has a computational complexity of $O(d^2(S + M))$, which scales more favorably than both the information gain maximization and EPMV strategies.

We unify the information gain (referred to as InfoGain), EPMV, and MCMV query selection methods under a single framework described in Algorithm~\ref{algo:utility}. At each round of querying, a pair is selected that maximizes a utility function $\eta(p,q)$ over a randomly downsampled pool of candidates pairs, with $\eta_0(p,q) \equiv I(W;Y_{i}|y^{i-1})$ for InfoGain and $\eta_1$ from \eqref{eq:equimv} and $\eta_2$ from \eqref{eq:MCMV_tradeoff} denoting the utility functions of EPMV and MCMV, respectively. We include a batch of posterior samples denoted by $\widetilde{W}$ as an input to $\eta_0$ and $\eta_1$ to emphasize their dependence on posterior sampling, and add mean and covariance inputs to $\eta_2$ since once these are estimated, MCMV requires no additional samples to select pairs. For all methods, we estimate the user point as the mean of the sample batch since this is the MMSE estimator.

\section{Results}
\label{sec:sim}
To evaluate our approach, we constructed a realistic embedding (from a set of training user-response triplets)
consisting of multidimensional item points and simulated our pairwise search methods over randomly generated preference points and user responses\footnote{Code available at \url{https://github.com/siplab-gt/pairsearch}}. We constructed an item embedding
of the Yummly \texttt{Food-10k} dataset
of~\cite{wilber2015learning,wilber2014cost},
consisting of 958,479 publicly available triplet comparisons assessing relative similarity among 10,000 food items.
The item coordinates are derived from the crowdsourced
triplets using the popular probabilistic multidimensional scaling algorithm of~\cite{tamuz_adaptively_2011}
and the implementation obtained
from the NEXT project\footnote{\url{http://nextml.org}}.

\subsection{Methods comparison}
We compare InfoGain, EPMV, and MCMV as described in Algorithm \ref{algo:utility} against several baseline methods: \\
\textbf{Random:} pairs are selected uniformly at random and user preferences are estimated as the posterior mean. \\
\textbf{GaussCloud-$\mathbf{Q}$:} pairs are chosen to approximate a Gaussian point cloud around the preference estimate that shrinks dyadically over $Q$ stages, as detailed in \cite{massimino2018you}. \\
\textbf{ActRank-$\mathbf{Q}$:} pairs are selected that intersect a feasible region of preference points and queried $Q$ times; a majority vote is then taken to determine a single response, which is used with the pair hyperplane to further constrain the feasible set \cite{jamieson2011active}. Since the original goal of the algorithm was to rank embedding items rather than estimate a continuous preference point, it does not include a preference estimation procedure; in our implementation we estimate user preference as the Chebyshev center of the feasible region since it is the deepest point in the set and is simple to compute \cite{boyd2004convex}.

In each simulation trial, we generate
a point $W$ uniformly at random from the hypercube $[-1,1]^d$ and collect paired comparisons using
the item points in our embedding according to the methods described above. The response probability
of each observation follows \eqref{eq:response} (referred to herein as the ``logistic'' model),
using each of the three schemes for choosing $k_{pq}$
described in \eqref{eq:k1}--\eqref{eq:k3}.
In each scheme we optimized the value of $k_0$ over the set of training triplets via maximum-likelihood estimation according to the logistic model. We use the Stan Modeling Language~\cite{carpenter2017stan}
to generate posterior samples when required, since our model is LCC and therefore is particularly
amenable to Markov chain Monte Carlo methods~\cite{brooks2011handbook}.

Note that unlike GaussCloud-$Q$ and ActRank-$Q$, the Random, InfoGain, EPMV, and MCMV methods directly exploit a user response model in the selection of pairs and estimation of preference points, which can be advantageous when a good model of user responses is available. Below we empirically test each method in this \emph{matched} scenario, where the noise type (logistic) and the model for $k_{pq}$ (e.g., ``constant'', ``normalized'', or ``decaying'') are revealed to the algorithms. We also test a \emph{mismatched} scenario by generating response noise according to a non-logistic response model while the methods above continue to calculate the posterior as if the responses were logistic. Specifically, we generate responses according to a ``Gaussian'' model
\[y_i = \operatorname{sign}(k_{pq}(a_i^T w - b_i) + Z)\qquad Z\sim \mathcal{N}(0,1)\]
where $k_0$ \emph{and} the model for $k_{pq}$ are selected using maximum-likelihood estimation on the training triplets.

\begin{figure*}[htbp]
    \def\fr{0.5}
    \def\fh{1.77}
    \centering
    \begin{subfigure}[t]{0.45\textwidth}
        \centering
        \includegraphics[height=\fh in]{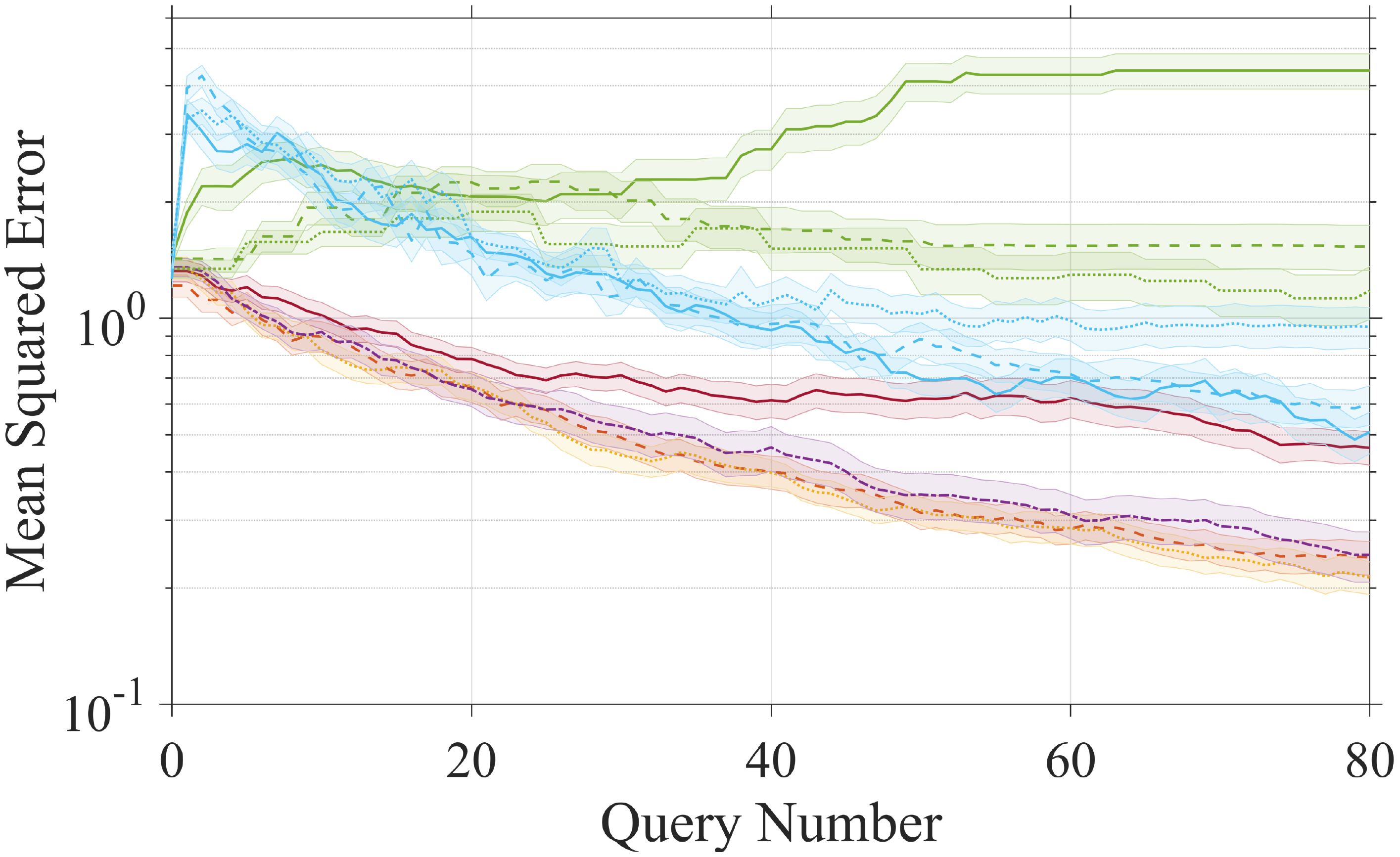}
        \caption{Estimation error: matched\\logistic noise, $d = 4$}
        \label{fig:matched_MSE}
    \end{subfigure}%
    \begin{subfigure}[t]{0.55\textwidth}
        \centering
        \includegraphics[height=\fh in]{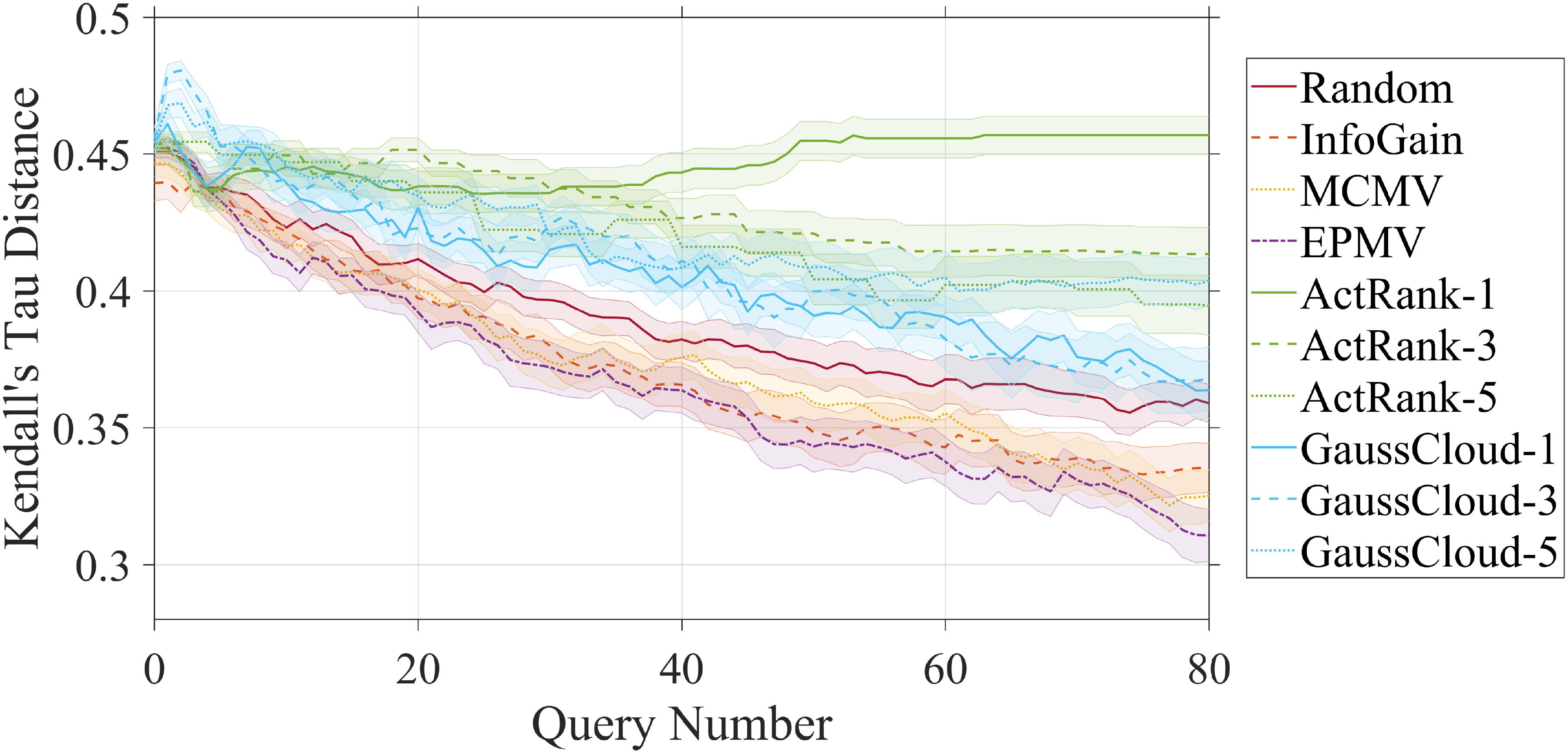}
        \caption{Ranking performance: matched\\logistic noise, $d = 4$}
        \label{fig:matched_KT}
    \end{subfigure}\\
    \begin{subfigure}[t]{0.45\textwidth}
        \centering
        \includegraphics[height=\fh in]{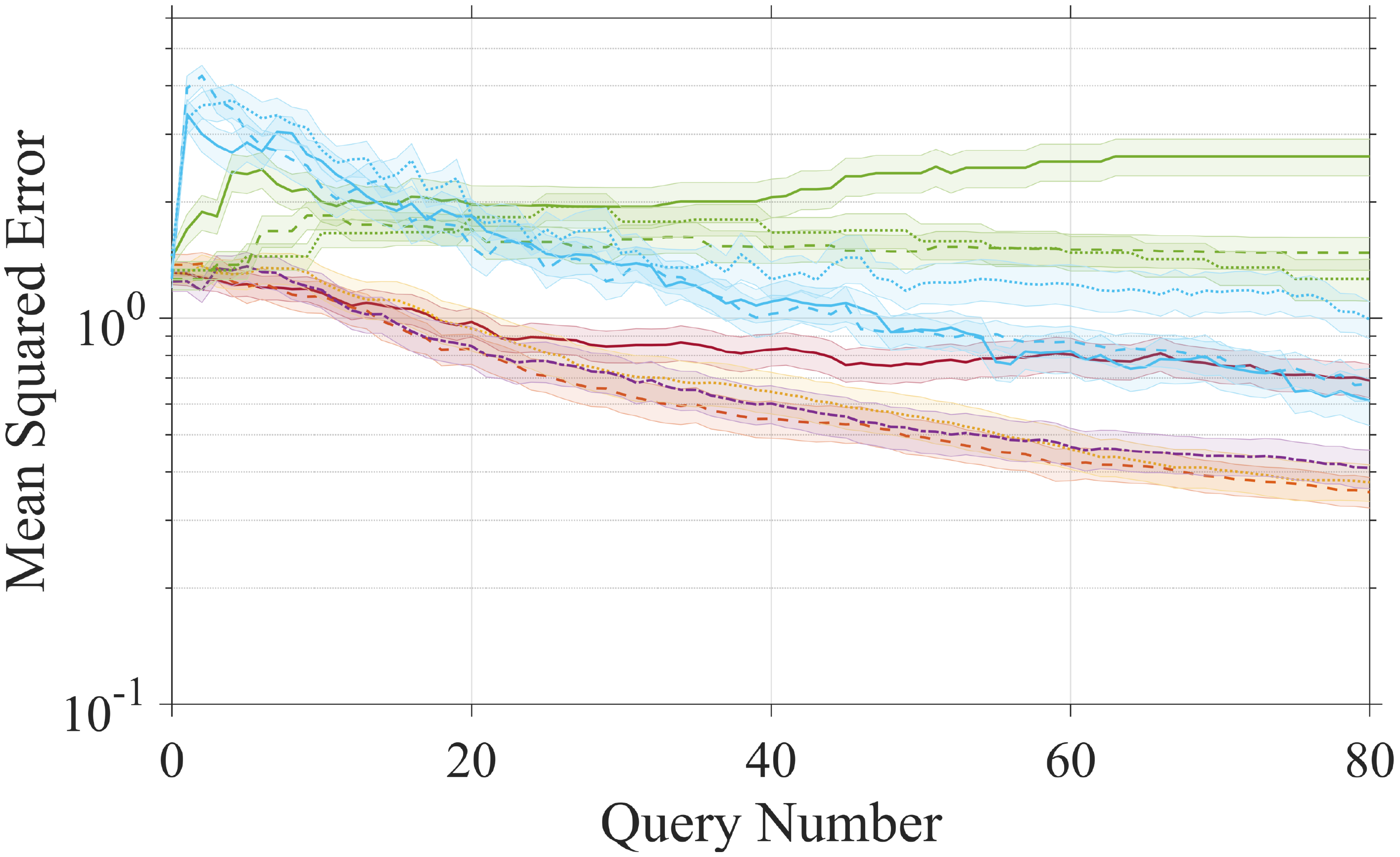}
        \caption{Estimation error: mismatched\\Gaussian noise, $d = 4$}
        \label{fig:mismatched_MSE}
    \end{subfigure}%
    \begin{subfigure}[t]{0.55\textwidth}
        \centering
        \includegraphics[height=\fh in]{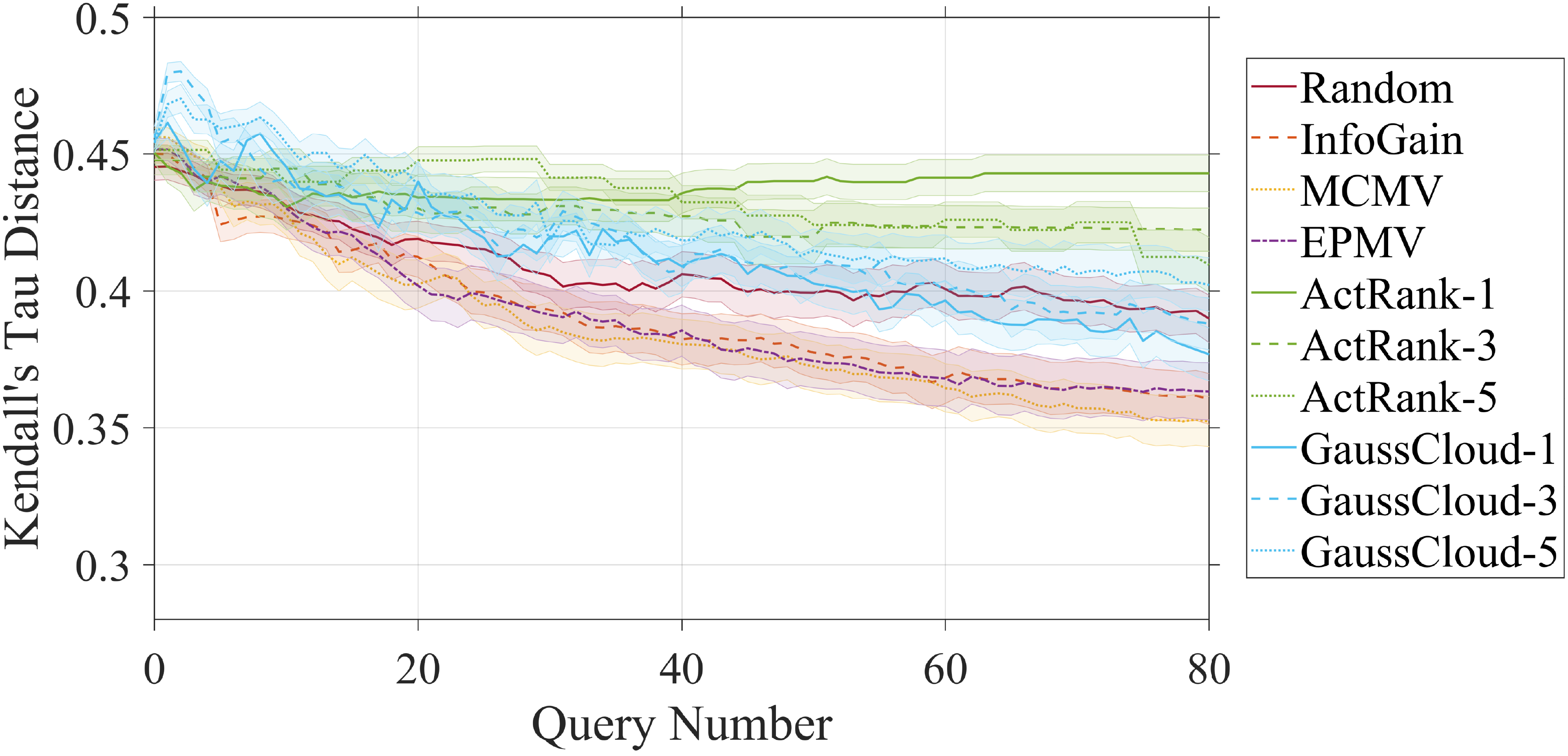}
        \caption{Ranking performance: mismatched\\Gaussian noise, $d = 4$}
        \label{fig:mismatched_KT}
    \end{subfigure}
    \caption{Performance evaluation over 80 simulated search queries, averaged over 50 trials per method and plotted with $\pm$ one standard error. (Left Column) $\MSE$. (Right Column) for each trial, a batch of 15 items was uniformly sampled without replacement from the dataset, and the normalized Kendall's Tau distance (lower distance is better) was calculated between a ranking of these items by distance to the ground truth preference point and a ranking by distance to the estimated point. To get an unbiased estimate, this metric was averaged over 1000 batches per trial, and error bars calculated with respect to the number of trials. (Top Row) ``normalized'' logistic model with matching noise in $d$ = 4. (Bottom Row) ``decaying'' logistic model with mismatched Gaussian ``normalized'' noise in $d$ = 4. Additional plots testing a wider selection of parameters are available in the supplement. Overall, our new strategies (EPMV, MCMV) outperform existing methods and also perform comparably to information gain maximization (InfoGain), which they were designed to approximate.}
    \label{fig:results}
\end{figure*}

\subsection{Mean squared error evaluation}
The left column of Figure~\ref{fig:results} plots the $\MSE$ of each method's estimate with respect to the ground truth location over the course of a pairwise search run. In the matched model case of Figure~\ref{fig:matched_MSE}, our strategies outperform Random, ActRank-$Q$, and GaussCloud-$Q$ for multiple values of $Q$ by a substantial margin. Furthermore, both of our strategies performed similarity to InfoGain, corroborating their design as information maximization approximations. Note that Random outperforms the other baseline methods, supporting the use of Bayesian estimation in this setting (separately from the task of active query selection). Although mismatched noise results in decreased performance overall in Figure~\ref{fig:mismatched_MSE}, the same relative trends between the methods as in Figure~\ref{fig:matched_MSE} are evident.

\subsection{Item ranking evaluation}
We also consider each method's performance with respect to ranking embedding items in relation to a preference point. For each trial, a random set of 15 items is sampled from the embedding without replacement and ranked according to their distance to a user point estimate. This ranking is compared to the ground truth ranking produced by the true user point by calculating a normalized Kendall's Tau distance, which is 0 for identical rankings and 1 for completely discordant rankings \cite{jamieson2011active}. This metric measures performance in the context of a recommender system type task (a common application of preference learning) rather than solely measuring preference estimation error. This metric is depicted in the right column of Figure~\ref{fig:results}, for the matched model case in \ref{fig:matched_KT} and mismatched case in \ref{fig:mismatched_KT}. The same trends as observed in $\MSE$ analysis occur, with our strategies performing similarly to InfoGain and outperforming all other methods. This is a particularly noteworthy result in that our method produces more accurate rankings than ActRank-$Q$, which to our knowledge is the state-of-the-art method in active embedding ranking.

\subsection{Discussion}

Our simulations demonstrate that both InfoGain approximation methods, EPMV and
MCMV, significantly outperform the state-of-the-art techniques in active preference estimation in the context of low-dimensional item embeddings with noisy user responses, and perform similarity to InfoGain, the method they were designed to approximate. This is true even when generating noise according to a different model than the one used for Bayesian estimation. These empirical results support the theoretical connections between EPMV, MCMV, and InfoGain presented in this study, and suggest that the posterior volume reduction properties of EPMV may in fact allow for $\MSE$ reduction guarantees.

These results also highlight the attractiveness of MCMV, which proved to be a top performer in embedding preference learning yet is computationally efficient and simple to implement. This technique may also find utility as a subsampling strategy in supervised learning settings with implicit pairwise feedback, such as in \cite{wu2017large}. Furthermore, although in this work pairs were drawn from a fixed
embedding, MCMV is easily adaptable to
continuous item spaces that allow for
generative construction of new items to compare.
This is possible in some applications,
such as facial composite generation for criminal
cases~\cite{frowd2011evolving}
or in evaluating foods and beverages,
where we might be able to generate nearly
arbitrary stimuli based on the ratios of ingredients~\cite{ventura2011sugar}.

\section*{Acknowledgements}
We thank the reviewers for their useful feedback and comments, as well as colleagues for insightful discussions including Matthieu Bloch, Yao Xie, Justin Romberg, Stefano Fenu, Marissa Connor, and John Lee. This work is supported by NSF grants CCF-1350954 and CCF-1350616, ONR grant N00014-15-1-2619, and a gift from the Alfred P. Sloan Foundation.

\bibliographystyle{IEEEtran}
\bibliography{refs}

\clearpage
\appendix
\section{Supplementary material}

First, we begin with an additional lemma:
\begin{lemma}\label{lem:lccub}
Let $X_i$ be a marginal distribution of $W$.
The density of $X_i$ is then
\[
    p_{X_i|y^{i}}(x) = \frac{1}{\sigma_i}p_{Z_i}\left(\frac{X_i-\E[X_i|y^i]}{\sigma_i}\right) \leq \frac{1}{\sigma_i},
\]
where
$\sigma_i = \sqrt{\E[(X_i - \E[X_i|y^{i}])^2|y^i]}$
and
$Z_i = \frac{X_i-\E[X_i|y^i]}{\sigma_i}$.
\end{lemma}

\begin{proof}
Since $X_i$ is a marginal of a log-concave distribution,
$X_i$ is also log-concave. Furthermore,
$Z_i$ is a zero-mean, unit-variance (i.e., isotropic) log-concave random variable with density $p_{Z_i}(z)$.
Then Lemma~\ref{lem:lccub}
follows because one-dimensional isotropic log-concave
densities are upper bounded by one~\cite{lovasz2007geometry}.
\end{proof}

A direct consequence of Lemma~\ref{lem:lccub}
is that for any $a > 0$,
\begin{align*}
    P(\abs{X_i} < a\,|\,y^{i}) = \int_{-a}^a p_{X_i|y^{i}}(x) dx
    \leq \frac{1}{\sigma_i} \int_{-a}^a dx
    \leq \frac{2a}{\sigma_i},
\end{align*}
implying that
\begin{equation}
    P(\abs{X_i} \geq a\,|\,y^{i}) \geq 1- \frac{2a}{\sigma_i}.
    \label{eq:tail_lower}
\end{equation}

\subsection{Proof of Lemma \ref{lem:LCC_ent_Lower}}

\begin{proof}
Letting $\Sigma_W$ denote the $d\times d$ covariance matrix
of random vector $W\in\R^d$, from Theorem~8.6.5 in~\cite{cover2012elements},
we have the upper bound
\begin{equation}
    h(W) \le \frac12 \log_2((2\pi e)^d|\Sigma_W|).
    \label{eq:ent_upper}
\end{equation}
Now assume the distribution $P_W$ of $W$ is log-concave, let
$W_1,W_2\sim P_W$ be i.i.d.\ and let $\widetilde W:= W_1 - W_2$.
Let $p_{\widetilde W}$ and
$p_W$ denote the respective densities of $\widetilde W$ and $W$.
We have by Proposition~3.5 of~\cite{saumard2014log},
for all $z\in\R^d$,
\begin{equation}
    p_{\widetilde W}(z) = p_W(z) \star p_W(-z),
    \label{eq:lem21a}
\end{equation}
where $\star$ is the convolution operator,
is also log-concave.
Since covariances add for independent random vectors,
$\Sigma_{\widetilde W} = 2\Sigma_W$.

By Theorem~4 of~\cite{marsiglietti2018lower}, for $d\ge 2$
\[
    h(\widetilde W) \ge \frac d2 \log_2
    \frac{|\Sigma_{\widetilde{W}}|^{1/d}}{c(d)},
\]
where $c(d) = e^2d^2/(4\sqrt2(d+2))$.
From Corollary~2.3 of~\cite{bobkov2013problem},
\[
    h(\widetilde W) = h(W_1-W_2) \le h(W) + d \log_2 e,
\]
which implies
\begin{align}
    h(W) \ge h(\widetilde W) - d \log_2 e
    &\ge \frac d2 \log_2 \frac{|\Sigma_{\widetilde W}|^{1/d}}{c(d)} - d \log_2 e
    \ge \frac d2 \log_2\frac{|2\Sigma_W|^{1/d}}{e^2 c(d)}
    \label{eq:lem21b}
\end{align}

The result follows combining \eqref{eq:ent_upper} and \eqref{eq:lem21b}.
\end{proof}

\subsection{Proof of Theorem \ref{thm:lower1}}
\begin{align}
\E_{Y^i}[h_i(W)] = h_0(W) - \sum_{j=1}^i I(W;Y_j|Y^{j-1})
\geq - i
\end{align}
from the chain rule for mutual information with $h_0(W)=0$ and $I(W;Y_j|Y^{j-1}) \leq 1$~\cite{cover2012elements}, and
\begin{align}
\E_{Y^i}[h_i(W)] &\leq \frac{1}{2} \E_{Y^i}\log_2((2\pi e)^d \abs{\Sigma_{W|Y^i}})
\\ &\leq \frac{1}{2} \log_2((2\pi e)^d \abs{\E_{Y^i}\Sigma_{W|Y^i}})
\end{align}
from Lemma \eqref{lem:LCC_ent_Lower} with Jensen's inequality and the concavity of $\log\abs{A}$ for any matrix $A$ in the positive definite cone~\cite{boyd2004convex}. Rearranging, we have
\begin{align}
    \frac{2^{-2i}}{(2 \pi e)^d} &\leq \abs{\E_{Y^i}\Sigma_{W|Y^i}}
    \leq \frac{\Tr\left({\E_{Y^i}[\Sigma_{W|y^i}]}\right)^d}{d^d} \label{eq:AMGM}\\
    &= \frac{(\E_{W,Y^i}[\norm{W - \E[W|Y^i]}_2^2])^d}{d^d}\label{eq:lin_exp_tr}\\
    &\leq \frac{(\E_{W,Y^i}[\norm{W - \widehat{W}_i}_2^2])^d}{d^d}
\end{align}
where \eqref{eq:AMGM} is from the AM--GM inequality, \eqref{eq:lin_exp_tr} is due to the linearity of trace and expectation, and the last inequality is due to that fact that expected value is the MMSE estimator, from which the $\MSE$ lower bound follows.

\subsection{Proof of Proposition \ref{prop:milb}}

\begin{proof}
Consider the `equiprobable' query scheme, with $P(Y_i=1|y^{i-1})=\frac{1}{2}$ for hyperplane query given by weights $a_i$, threshold $\tau_i$, and noise constant $k$. Letting $X_i = a_i^T W - \tau_i$, we have
\begin{align*}
    I(W;Y_i|y^{i-1}) &= H(Y_i|y^{i-1}) - H(Y_i|y^{i-1},W)\\
    &= H(Y_i|y^{i-1}) - H(Y_i|y^{i-1},W,X_i)\intertext{since $X_i$ is a deterministic function of $W$}
    &= H(Y_i|y^{i-1}) - H(Y_i|y^{i-1},X_i) \intertext{since $p(Y_i|y^{i-1},W,X_i) = p(Y_i|y^{i-1},X_i)$}
    &= I(X_i;Y_i|y^{i-1}).
\end{align*}
Revisiting mutual information, we have
\begin{align}
    I(X_i;Y_i|y^{i-1}) &= \E\left[\log_2 \frac{p(Y_i|X_i,y^{i-1})}{p(Y_i|y^{i-1})}\right]\\
    &=E_{X_i} [(1-h_b(f(k X_i)))\,|\,y^{i-1}]
    \\
    &= E_{X_i}[(1-h_b(f(k \abs{X_i})))|y^{i-1}]\intertext{since $1-h_b(f(k X_i))$ is symmetric. From Markov's inequality with $1-h_b(f(k \abs{X_i}))$ being monotonically increasing, for any $a > 0$,}
    &\geq (1-h_b(f(k a))) P(\abs{X} > a\,|\,y^{i-1})\\
    \text{(from \eqref{eq:tail_lower}) }
    &\geq (1-h_b(f(k a))) \left(1- \frac{2a}{\sigma_i}\right) \\
    &= \left(1-h_b\left(f\left(\frac{k c \sigma_i}{2}\right)\right)\right) \left(1- c\right)
\end{align} by letting $a = \frac{c \sigma_i}{2}$ for any $0 \leq c \leq 1$
\end{proof}

\subsection{Proof of Theorem~\ref{thm:guarantee}}
\paragraph{Entropy properties:} Let $h(W|y^i)$ denote the posterior entropy after observing $i$ queries. With a uniform prior distribution over the hypercube $[-\frac{1}{2},\frac{1}{2}]$, we have that $h(W|y^0) = 0$ and $h(W|y^i) \leq 0$ for $\forall i$ since the uniform distribution maximizes entropy over this bounded space.

After query $i$, let the eigenvalues of the posterior covariance matrix be denoted in decreasing order as $\lambda_1 \geq \lambda_2 \dots \geq \lambda_d$. In the equiprobable, max-variance scheme, query $a_i$ is in the direction of maximal eigenvector, so the product of the noise constant and query standard deviation at iteration $i$ is given by $k \sqrt{a_i^T\Sigma_{W|y^i} a_i} =  k \norm{a_i} \sqrt{\lambda_1} \geq k_\text{min} \sqrt{\lambda_1}$. From the monotonicity of the mutual information lower bound on equiprobable queries, we have
\begin{equation}
    I(W;Y_i|y^{i-1}) \geq L_{c,k_\text{min}}(\sqrt{\lambda_1})
\end{equation}

From rearranging terms in Lemma \ref{lem:LCC_ent_Lower} along with $\abs{\Sigma_{W|y^i}}=\prod_{i=1}^d \lambda_i$, we have
\begin{align}
    \frac{2^{2h(W|y^i)}}{(2\pi e)^d} &\leq \abs{\Sigma_{W|y^i}}=\prod_{i=1}^d \lambda_i \leq \lambda_1^d\\
    \implies \lambda_1 &\geq \frac{2^{\frac{2h(W|y^i)}{d}}}{2\pi e}
\end{align}

For compactness of notation, let
\begin{equation}
    \tilde{L}_{c,k_\text{min}}(h) = L_{c,k_\text{min}}\left(\frac{2^{\frac{h}{d}}}{\sqrt{2\pi e}}\right)
\end{equation}
Since $L_{c,k_\text{min}}$ is monotonically increasing, we have
\begin{equation}
    I(W;Y_i|y^{i-1}) \geq \tilde{L}_{c,k_\text{min}}(h(W|y^i))
\end{equation}

Combined with the 1 bit upper bound on mutual information along with $I(W;Y_i|y^{i-1}) = h(W|y^{i-1}) - \E_{Y_{i}|y^{i-1}}[h(W|y^{i})]$, we have
\begin{align}
    h(W|y^{i-1}) - 1 &\leq \E_{Y_{i}|y^{i-1}}[h(W|y^{i})]
    \label{eq:ent_squeeze}
    \\ & \leq h(W|y^{i-1}) - \tilde{L}_{c,k_\text{min}}(h(W|y^{i-1}))
    \notag
\end{align}

To bound the entropy deviations from one measurement to the next, we need the following lemma:

\begin{lemma}\label{lem:technical}
For the equiprobable query scheme,
\[
\abs{h(W|y^{i}) - h(W|y^{i-1})} \leq \gamma(d)\quad\forall i\geq 0
\] where $\gamma(d) = 8d+\frac{d}{2}\log_2{(2\pi e d)}+1$.
\end{lemma}
The proof of Lemma~\ref{lem:technical} is highly technical
and so we relegate it to the end of the supplementary materials.

\paragraph{Martingale properties:}
We note our martingale argument is similar in style to
\cite{burnashev1974interval}.
Let $Z_i = -h(W|y^i)$. From the previous section we have $Z_0 = 0,~Z_i\geq0~\forall i\geq 0$, $\abs{Z_{i}-Z_{i-1}} \leq \gamma(d)$
from Lemma~\ref{lem:technical},
and $Z_{i-1} + \tilde{L}_{c,k_\text{min}}(-Z_{i-1}) \leq E_{Z_{i}|y^{i-1}}[Z_{i}] \leq Z_{i-1}+1$. Since $Z_{i-1}$ is a deterministic function of $y^{i-1}~\forall i$ along with the law of total expectation,
\begin{align*}
    \E[Z_{i}|Z_0,\dots,Z_{i-1}] &= \E_{Y^{i-1}|Z_0,\dots,Z_{i-1}}\E[Z_{i}|Z_0,\dots,Z_{i-1},y^{i-1}]\\
    &= \E_{Y^{i-1}|Z_0,\dots,Z_{i-1}}\E[Z_{i}|y^{i-1}]
\end{align*}
which implies
\begin{align*}
    \E[Z_{i}|Z^{i-1}] &\geq \E_{Y^{i-1}|Z_0,\dots,Z_{i-1}}[Z_{i-1} + \tilde{L}_{c,k_\text{min}}(-Z_{i-1})]\\
    &=Z_{i-1} + \tilde{L}_{c,k_\text{min}}(-Z_{i-1})
\end{align*}
and
\begin{align*}
    \E[Z_{i}|Z_0,\dots,Z_{i-1}] &\leq\E_{Y^{i-1}|Z_0,\dots,Z_{i-1}}[Z_{i-1} + 1]\\
    &=Z_{i-1} + 1
\end{align*}

Since $\tilde{L}_{c,k_\text{min}}(-Z_{i-1}) > 0$, we have $\E[Z_{i}|Z^{i-1}] \geq Z_{i-1}$. For all $i \geq 0$, $\abs{Z_{i}} < \infty$ since $\abs{Z_{i}} = \abs{Z_0 + \sum_{j=1}^{i} Z_{j}-Z_{j-1}} \leq \sum_{j=1}^{i} \abs{Z_{j}-Z_{j-1}} \leq i \gamma(d) < \infty$. Therefore, $Z_i$ is a submartingale.

Let $\tau > 0$ define a stopping threshold and corresponding stopping time $T=\min \{i:Z_i \geq \tau\}$ Considering $\E[Z_{i}|Z^{i-1}] \leq Z_{i-1} + 1$ and taking the expectation over $Z^{i-1}$ on both sides and expanding with the tower rule, we have
\begin{align*}
    \E[\E[Z_{i}|Z^{i-1}]] &\leq \E[Z_{i-1}] + 1\\
    \E[Z_{i}] &\leq \E\E[Z_{i-1}|Z^{i-2}] + 1\\
    \E[Z_{i}] &\leq \E[Z_{i-2}] + 1 + 1\\
    &\dots\\
    \E[Z_{i}] &\leq i
\end{align*}
which implies
\begin{equation*}
    T \geq \E[Z_T] \geq \tau
\end{equation*}
where the last inequality follows by definition, so $\E[T] \geq \tau$. Note that this is true for \emph{any} query selection scheme since mutual information is always upper bounded by 1 bit.

To lower bound the expected stopping time, observe $\tilde{L}_{c,k_\text{min}}(-z)$ is monotonically decreasing in $z$, and $Z_i \leq \tau$ for $i < T$, so we have in this range that $\tilde{L}_{c,k_\text{min}}(-Z_i) > \tilde{L}_{c,k_\text{min}}(-\tau)$. Using this fact, we construct a separate submartingale that equals $Z_i$ up to and including the stopping time and has the same properties listed above. Specifically, let
\begin{equation}
    U_i = \begin{cases}
    Z_i&i\leq T\\
    U_{i-1}+\tilde{L}_{c,k_\text{min}}(-\tau) & i> T.
    \end{cases}
\end{equation}

Clearly for $i \leq T$, $U_i = Z_i$, and if $T_U$ is defined as $T_U = \min \{i:U_i \geq \tau\}$, by observation $T_U = T$. $U_i$ also satisfies $\abs{U_{i}-U_{i-1}}<\gamma(d)$, and $U_{i-1} + \tilde{L}_{c,k_\text{min}}(-\tau) \leq E[U_{i}|U^{i-1}] \leq U_{i-1}+1$.

We have
\begin{align}
    E[U_{i}|U^{i-1}] \geq U_{i-1} + \tilde{L}_{c,k_\text{min}}(-\tau)\\
    \frac{E[U_{i}|U^{i-1}]}{\tilde{L}_{c,k_\text{min}}(-\tau)} \geq \frac{U_{i-1}}{\tilde{L}_{c,k_\text{min}}(-\tau)} + 1\\
    \frac{E[U_{i}|U^{i-1}]}{\tilde{L}_{c,k_\text{min}}(-\tau)} - i\geq \frac{U_{i-1}}{\tilde{L}_{c,k_\text{min}}(-\tau)} - (i-1)
\end{align}

We then have a submartingle given by $U^{(\text{sub})}_i = \frac{U_i}{\tilde{L}_{c,k_\text{min}}(-\tau)} - i$. 

Assume for the time being that the optional stopping theorem can be applied to this submartingale (proved in the sequel)---for any stopping time $S$ satisfying $S \leq T$, $\E[U^\text{sub}_S] \leq \E[U^\text{sub}_T]$. Specifically, if $\tau_S$ is a stopping threshold satisfying $\tau_S \leq \tau$ such that $S = \min\{i:U_i \geq \tau_S\}$, then (for brevity, letting $l(u) = \tilde{L}_{c,k_\text{min}}(-u)$)
\begin{equation}
    \frac{\E[U_S]}{l(\tau)}- \E[S] \leq \frac{\E[U_T]}{l(\tau)}- \E[T]
\end{equation}
which implies
\begin{align}
    \frac{\E[U_S]}{l(\tau_S)}- \E[S] &= \frac{l(\tau)}{l(\tau_S)}\left[\frac{\E[U_S]}{l(\tau)}- \E[S]\right] -
    \left(1 - \frac{l(\tau)}{l(\tau_S)}\right)\E[S]
    \\ &\leq \frac{l(\tau)}{l(\tau_S)}\left[\frac{\E[U_T]}{l(\tau)}- \E[T]\right] - \left(1 - \frac{l(\tau)}{l(\tau_S)}\right)\E[S]
\end{align}
More generally, let $\Delta > 0$ be given and set stopping threshold $\tau_i = i \Delta$, with corresponding stopping time $T_i$. Define $P_i = \frac{U_{T_i}}{l(\tau_i)}-T_i$. Letting $r_i = \frac{l(\tau_i)}{l(\tau_{i-1})}$ and letting $T = T_i$ and $S=T_{i-1}$, by rearranging the above we have
\begin{equation}
    \E[P_i] \geq \frac{\E[P_{i-1}]}{r_i}+\frac{\left(1-r_i\right)}{r_i}\E[T_{i-1}]
\end{equation}
Noting that $\E[T_0]=0$ since a threshold of $\tau_0$ results in stopping at $T_0 = 0$ and $E[P_0]=\frac{U_{T_0}}{l(\tau_0)}-\E[T_0]=0$, we continue this bound recursively
\begin{align*}
\E[P_i] &\geq \frac{\E[P_{i-2}]}{r_i r_{i-1}}+\frac{\left(1-r_{i-1}\right)}{r_i r_{i-1}}\E[T_{i-2}]
+\frac{\left(1-r_i\right)}{r_i}\E[T_{i-1}]
\dots
\\ &= \sum_{j=1}^{i-1} \frac{1-r_{j+1}}{\prod_{k=j+1}^i r_k}\E[T_{j}]\\
&= \sum_{j=1}^{i-1} \frac{l(\tau_j)-l(\tau_{j+1})}{l(\tau_i)}\E[T_{j}]\intertext{since $\prod_{k=j+1}^i r_k = \frac{l(\tau_i)}{l(\tau_{i-1})}\frac{l(\tau_{i-1})}{l(\tau_{i-2})}\dots\frac{l(\tau_{j+1})}{l(\tau_{j})}=\frac{l(\tau_i)}{l(\tau_{j})}$}
&= \frac{1}{l(\tau_i)}\sum_{j=1}^{i-1} \frac{l(j \Delta)-l(j \Delta + \Delta)}{\Delta}\Delta\E[T_{j}]\\
&\geq \frac{1}{l(\tau_i)}\sum_{j=1}^{i-1} \frac{l(\tau_j)-l(\tau_j + \Delta)}{\Delta}\tau_j \Delta \intertext{since $\E[T_j] \geq \tau_j = j \Delta$. Now let $\tau > 0$ be given (with corresponding stopping time defined as $T$) and let $\Delta \to 0$, choosing $i$ appropriately such that $\tau = \tau_i = i \Delta$}
&\geq -\frac{1}{l(\tau)}\int_{0}^{\tau} \biggl(\frac{d}{dx}l(x) \biggr) x dx\\
&=\frac{1}{l(\tau)} \int_{0}^{\tau} l(x) dx - \tau\\
&\implies \frac{\E[U_{T}]}{l(\tau)}-\E[T] \geq \frac{1}{l(\tau)} \int_{0}^{\tau} l(x) dx - \tau\\
\implies \E[T] &\leq \tau + \frac{\E[U_{T}]}{l(\tau)} - \frac{1}{l(\tau)} \int_{0}^{\tau} l(x) dx\\
&\leq \tau + \frac{\tau+1}{l(\tau)} - \frac{1}{l(\tau)} \int_{0}^{\tau} l(x) dx\intertext{since $\E[U_{T}]=\E[\E[U_{T}|U^{T-1}]]\leq \E[U_{T-1}] + 1 \leq \tau + 1$.}
\end{align*}

All together we have
\begin{equation}
    \tau \leq \E[T] \leq \tau + \frac{\tau+1}{l(\tau)} - \frac{1}{l(\tau)} \int_{0}^{\tau} l(x) dx \label{eq:martingale_stopping_upper_lower}
\end{equation}

Now, suppose we'd like to stop the algorithm when the posterior covariance determinant crosses below a threshold, corresponding to a low posterior volume. Denote this threshold as $\varepsilon$, and define the stopping time $T_\varepsilon$ as $\min \{i: \abs{\Sigma_{W|y^i}}^\frac{1}{d} < \varepsilon\}$. By rearranging the upper bound in Lemma \ref{lem:LCC_ent_Lower} we have the necessary condition
\begin{equation}
    h_i(W) \leq \frac{d}{2} \log_2(2\pi e \varepsilon)
\end{equation}
Letting $\tau_1 = \frac{d}{2} \log_2(\frac{1}{2\pi e \varepsilon})$ be the entropic stopping threshold with stopping time $T_1$, from \eqref{eq:martingale_stopping_upper_lower} this results in (with $\E[T_\varepsilon] \geq \E[T_1]$ since this is a necessary condition)

\begin{equation}
    \E[T_\varepsilon] \geq \E[T_1] \geq \tau_1 \label{eq:cov_stop_lower}
\end{equation}

Similarly, by rearranging the lower bound in Lemma \ref{lem:LCC_ent_Lower} we observe that a sufficient condition for this stopping criterion is
\begin{equation}
h_i(W) \leq \frac{d}{2} \log_2\frac{2 \eps}{e^2 c_d}
\end{equation}
where $c_d = (e^2 d^2)/(4\sqrt{2} (d+2))$. Letting $\tau_2=\frac{d}{2} \log_2\frac{e^2 c_d}{2 \varepsilon}$ be the entropic stopping threshold with stopping time $T_2$, we have from \eqref{eq:martingale_stopping_upper_lower} (with $\E[T_\varepsilon] \leq \E[T_2]$ since this is only a sufficient condition):
\begin{equation}
    \E[T_\varepsilon] \leq \E[T_2] \leq \tau_2 + \frac{\tau_2+1}{l(\tau_2)} - \frac{1}{l(\tau_2)} \int_{0}^{\tau_2} l(x) dx
\end{equation}

Combining these, we have the theorem result.

\paragraph{Verifying optional stopping theorem:} Consider a submartingale of the form $P_i = \frac{Q_i}{C} - i$ for some $C > 0$, where $Q_i$ is also a submartingale satisfying $Q_i = 0$, $Q_i \geq 0$ for $i \geq 0$, and $\abs{Q_{i+1} - Q_i} \leq B$ for some $B > C > 0$. This implies
\begin{align*}
    \abs{P_{i} - P_{i-1}} &= \left\vert\frac{Q_{i}}{C} - i - \frac{Q_{i-1}}{C} + (i-1)\right\vert \\
    &= \frac{\left\vert Q_{i} - Q_{i-1} - C \right\vert}{C}\\
    &\leq \frac{\left\vert Q_{i} - Q_{i-1}\right\vert}{C} + 1\\
    &\leq \frac{B}{C}+1 \eqqcolon B' < \infty
\end{align*}
Let stopping time $T_Q$ be defined as $\min \{i: Q_i > \tau\}$ for some threshold $0< \tau < \infty$. This implies a stopping time on $P_i$ given by $T_P = \min \{i: P_i > \frac{\tau}{C}-i\}$, with $T\coloneqq T_Q = T_P$. We have from Theorem 5.2.6 of \cite{durrett2010probability} that $P_{T \wedge i}$ and $Q_{T \wedge i}$ are also submartingales.

Consider $\sup \E Q_{T \wedge i}^+ = \sup \E Q_{T \wedge i} \leq \tau + B< \infty$, by definition. From Theorem 5.2.8 of \cite{durrett2010probability}, as $i \to \infty$, $Q_{T \wedge i}$ converges a.s.\ to a limit $Q$ with $\E|Q| < \infty$ (and hence $|Q| < \infty$ a.s.). This also implies $|Q_{T \wedge i}| \overset{\text{a.s.}}{\to}|Q|$.

Similarly, $\sup \E P_{T \wedge i}^+=\sup \E\left[ \left\{\frac{Q_{T \wedge i}}{C} - (T \wedge i)\right\}^+\right]\leq \sup \E\left[\frac{Q_{T \wedge i}^+}{C}\right] \leq \frac{\tau + B}{C} < \infty$, so as $i \to \infty$, $P_{T \wedge i}$ converges a.s.\ to a limit $P$ with $\E|P| < \infty$ (and hence $|P| < \infty$ a.s.). This also implies $|P_{T \wedge i}| \toas|P|$.

We have
\begin{align*}
    T \wedge i &= \left\vert (T \wedge i) - \frac{Q_{T \wedge i}}{C} + \frac{Q_{T \wedge i}}{C}\right\vert\\
    &\leq \left\vert (T \wedge i) - \frac{Q_{T \wedge i}}{C}\right\vert + \frac{\left\vert Q_{T \wedge i}\right\vert}{C}\\
    &= \abs{P_{T \wedge i}} + \frac{\left\vert Q_{T \wedge i}\right\vert}{C}
\end{align*}
Since the right side converges a.s.\ to a limit $|P| + \frac{|Q|}{C} \eqqcolon L$ and $L < \infty$ a.s., for all large enough $i$, $T \wedge i < L$ a.s. which implies $T < L$ a.s.\ and therefore $\E[T] < \infty$. Combining this fact with $\abs{P_{i+1} - P_i} \leq B'$, Theorem 5.7.5 of \cite{durrett2010probability} gives that $P_{T \wedge i}$ is uniformly integrable. Then, from Theorem 5.7.4 of \cite{durrett2010probability}, for any stopping time $L \leq T$, $\E[P_L] \leq \E[P_T]$.

\subsection{Proof of Theorem \ref{thm:stoppingorder}}
To lower bound the complexity of $T_\varepsilon$, we substitute the definition of $\tau_1$ into \eqref{eq:cov_stop_lower}, which is true for any query scheme:
\begin{align}
    \E[T_\varepsilon] &\geq \frac{d}{2} \log_2\left(\frac{1}{2\pi e \varepsilon}\right)\\
    \implies \E[T_\varepsilon] &= \Omega\left(d \log \frac{1}{\varepsilon}\right)
\end{align}

To upper bound the complexity of $T_\varepsilon$, note that $\tau_2 - \frac{1}{l(\tau_2)} \int_{0}^{\tau_2} l(x) dx \leq 0$ from the mean value theorem, so $\E[T_\varepsilon] \leq \frac{\tau_2+1}{l(\tau_2)}$. Also note that
\begin{align}
    L_{c,k}(\sigma) &= \left(1-h_b\left(f\left(\frac{c k \sigma}{2}\right)\right)\right) \left(1- c\right) \notag \\
    &\geq \left(1-\sech\left(\frac{ck\sigma}{4}\right)\right)(1-c) \label{eq:sech_sub}\\
    &\geq \frac{c^2k^2\sigma^2}{32+c^2k^2\sigma^2}(1-c) \label{eq:cosh_sub}
\end{align}

where \eqref{eq:sech_sub} is from $h_b(p) \leq 2\sqrt{p(1-p)}$, and \eqref{eq:cosh_sub} is from $\sech(x) \leq \frac{2}{2+x^2}$.

Plugging in the definition for $\tau_2$ into $l(\tau_2)$ we have
\begin{equation}
    l(\tau_2) = L_{c,k_\text{min}}\left(\frac{2^{\frac{-\tau_2}{d}}}{\sqrt{2\pi e}}\right)=L_{c,k_\text{min}}\left(\sqrt{\frac{\varepsilon}{\pi e^3 c_d}}\right)
\end{equation}
so
\begin{equation}
    l(\tau_2) \geq \frac{c^2k_\text{min}^2}{32\pi e^3 c_d\frac{1}{\varepsilon}+c^2k_\text{min}^2}(1-c)
\end{equation}
which implies
\begin{align}
    \E[T_\varepsilon] &\leq \frac{\left(\frac{d}{2} \log_2\frac{e^2 c_d}{2 \varepsilon}+1\right)\left(32\pi e^3 c_d\frac{1}{\varepsilon}+c^2k_\text{min}^2\right)}{(1-c)c^2k_\text{min}^2}\\
    \implies \E[T_\varepsilon] &= O\left(d\log\frac{1}{\varepsilon} + \left(\frac{1}{\varepsilon k_\text{min}^2}\right) d^2\log\frac{1}{\varepsilon}\right)
\end{align}

\subsection{Proof of Proposition \ref{prop:meancut_MI_lower}}

\begin{proof}
We first bound $p_1 := P(Y=1)$.
Recall that for some fixed $k$, $f(x) = (1+e^{-kx})^{-1}$.  First note that
\begin{align*}
    \int_a^b f(x) dx = \int_a^b \frac{1}{k}\frac{ke^{kx}}{1+e^{kx}} dx
    = \frac{1}{k}\int_a^b \frac{u'}{u}dx
    = \frac{1}{k}\int_{u(a)}^{u(b)} \frac{1}{u} du
    = \frac{1}{k}\ln \frac{1+e^{kb}}{1+e^{ka}}.
\end{align*}
We have that $P(Y = 1) = \E[P(Y=1|X=x)]=\E[f(x)]$.
Note that $\forall x$, $(1+e^{-kx})\le 1$.  Then,
\begin{align*}
    p_1 = \E[f(x)] &= \int f(x)p_X(x)dx
        \\ &= \int_{x\le 0} f(x)p_X(x)dx + \int_{x > 0} f(x)p_X(x) dx
        \\ &\le \frac{1}{\sigma_X}\int_{-\infty}^0 f(x) dx
        + \int_{x>0} f(x)p_X(x)dx
        \\ &\le \frac{1}{\sigma_Xk}\ln\frac{1+e^{k0}}{1}+\int_{x>0} 1 p_X(x) dx
        \\ &\le  \frac{\ln 2}{\sigma_Xk} + P(X>0)
        \le \frac{\ln 2}{\sigma_X k} + 1-\frac{1}{e},
\end{align*}
where we use $p_X(x) \le 1/\sigma_X$ and
the final inequality follows from $P(X\le 0)\ge\frac{1}{e}$ for
zero-mean LCC $X$ \cite{lovasz2007geometry}. Using a similar argument it can be shown that
$\E[f(x)] \ge 1/e -\ln2/(\sigma_Xk)$.
Combining these, we have
\begin{equation}
    \frac1e-\frac{\ln2}{\sigma_X k}
    \le p_1
    \le 1-\left(\frac1e-\frac{\ln2}{\sigma_X k}\right).
    \label{eq:prop26_p1bound}
\end{equation}

Now we turn to lower bounding $I(X;Y) := H(Y) - H(Y|X)$.
The second term can be written
\begin{align}
    H(Y|X) = \E_X H(Y|X=x)
        = \int_{-\infty}^\infty h_b(f(x)) p_X(x) dx
        \le \frac{1}{\sigma_X} \int_{-\infty}^\infty h_b(f(x)) dx.
\end{align}
where the inequality follows from Lemma~\ref{lem:lccub}.
Since
\begin{align*}
    \MoveEqLeft H(Y|X = x)
    \\ &= -f(x)\log_2 f(x) - (1-f(x))\log_2(1-f(x))
    \\ &= \frac{1}{1+e^{-kx}}\log_2(1+e^{-kx})
    + \frac{e^{-kx}}{1+e^{-kx}}\log_2((1+e^{-kx})/e^{-kx})
    \\ &= \frac{1+e^{-kx}}{1+e^{-kx}}\log_2(1+e^{-kx})
    - \frac{e^{-kx}}{1+e^{-kx}}\log_2(e^{-kx})
    \\ &= \log_2(1+e^{-kx}) + \frac{kxe^{-kx}\log_2(e)}{1+e^{-kx}},
\end{align*}
which is an even function, we have (omitting details of the integration)
\begin{align}
    H(Y|X) \le \frac{2}{\sigma_X} \int_0^\infty \log_2(1+e^{-kx})
    + \frac{kxe^{-kx}\log_2(e)}{1+e^{-kx}} dx
    = \frac{\pi^2 (\log_2 e)}{3k
    \sigma_X
    } \label{eq:prop26a}
\end{align}

For the second term, note that $H(Y=1) = h_b(p_1)$. The binary entropy function
is symmetric about, and monotonically decreasing from $p=1/2$. Therefore,
\begin{equation}
    \label{eq:prop26b}
    H(Y) = h_b(p_1) \ge h_b\left(\frac1e - \frac{\ln2}{\sigma_Xk}\right)
\end{equation}

Combining \eqref{eq:prop26a} and \eqref{eq:prop26b} gives the desired result.
\end{proof}

\subsection{Proof of Lemma~\ref{lem:technical}}

\begin{proof}
Since $p(W|y^{i})$
is log-concave, and by Jensen's inequality,
\begin{align*}
-h(W|y^{i}) &= \E_{W|y^{i}} [\log_2 p(W|y^{i})]
\leq \log_2 p(\E[W|y^{i}]|y^{i}) \\
&\leq \log_2 \sup_w p(w|y^{i}).
\end{align*}
Without loss of generality, we may suppose $\E[W|y^{i}] = 0$, and let $V = \Sigma_{W|y^i}^{-\frac{1}{2}} W$ and $W \sim P_{W|y^i}$, such that $\E[V] = 0$ and $\E[V V^T] = \Sigma_{W|y^i}^{-\frac{1}{2}} \E[W W^T] \Sigma_{W|y^i}^{-\frac{1}{2}} = \Sigma_{W|y^i}^{-\frac{1}{2}}\Sigma_{W|y^i}\Sigma_{W|y^i}^{-\frac{1}{2}}=I$ and therefore $V$ is isotropic.
From \cite{klivans2009baum} we have that $p_V(v) \leq 2^{8d} d^{\frac{d}{2}}$. From the density of a linear transformation of a random variable we have
\[
p_{W|y^i}(w) = \frac{p_V(\Sigma_{W|y^i}^{-\frac{1}{2}}w)}{\abs{\Sigma_{W|y^i}^{\frac{1}{2}}}} \leq \frac{2^{8d} d^{\frac{d}{2}}}{\abs{\Sigma_{W|y^i}}^{\frac{1}{2}}}
.
\]
Therefore, for our query strategy we have (with $f_{i}(W)$ denoting the logistic response model for the query at iteration $i$)
\begin{align*}
\MoveEqLeft p(w|y^{i}) = p(w|y_{i} = y,y^{i-1})
\\ &= \frac{f_{i}(W)y + (1-f_{i}(W))(1-y)}{p(y_{i}=y|y^{i-1})} p(W|y^{i-1})
\\ &\leq \frac{(1)y + (1-(0))(1-y)}{p(y_{i}=y|y^{i-1})} p(W|y^{i-1})
\\ &= \frac{1}{p(y_{i}=y|y^{i-1})} p(W|y^{i-1})
\\ & \leq \frac{1}{p(y_{i}=y|y^{i-1})}\frac{2^{8d} d^{\frac{d}{2}}}{\abs{\Sigma_{W|y^{i-1}}}^{\frac{1}{2}}}
\\ \implies& \sup_w p(w|y^{i})
\leq \frac{1}{p(y_{i}=y|y^{i-1})}\frac{2^{8d} d^{\frac{d}{2}}}{\abs{\Sigma_{W|y^{i-1}}}^{\frac{1}{2}}},
\end{align*}
which implies
\begin{align*}
\log_2 \sup_w p(w|y^{i}) &\leq 8d + \frac{d}{2}\log_2 d
- \frac{1}{2} \log_2 \abs{\Sigma_{W|y^{i-1}}}
- \log_2(p(y_{i}=y|y^{i-1})),
\end{align*}
and hence
\begin{align*}
h(W|y^{i})
&\geq \frac{1}{2} \log_2 \abs{\Sigma_{W|y^{i-1}}} + \log_2(p(y_{i}=y|y^{i-1}))
- \left(8d + \frac{d}{2}\log_2 d\right)
\\ &\geq \frac{1}{2} \log_2 ((2\pi e)^d\abs{\Sigma_{W|y^{i-1}}})
- \frac{1}{2} \log_2 (2\pi e)^d + \log_2(p(y_{i}=y|y^{i-1}))
- \left(8d + \frac{d}{2}\log_2 d\right)\\
&\geq h(W|y^{i-1}) + \log_2(p(y_{i}=y|y^{i-1}))
- \left(8d + \frac{d}{2}\log_2 (2\pi e d)\right) \quad\text{from
\eqref{eq:ent_upper}}.
\end{align*}
For equiprobable queries $p(y_{i}=y|y^{i-1}) = 1/2$, and so
we have
\begin{equation}
h(W|y^{i-1}) - h(W|y^{i}) \leq \gamma(d).
\label{eq:martingale_deviation_lower_equi}
\end{equation}
where $\gamma(d) = 8d + \frac{d}{2}\log_2 (2\pi e d) + 1$.

To obtain the other direction, let $h_y^{i-1} = h(W|Y_{i}=y,y^{i-1})$, $y_m = \argmin_{y \in \{0,1\}} h_y^{i-1}$, $y_M = 1 - y_m$. Note that $h_{y_M}^{i-1} \geq h_{y_m}^{i-1}$. We have
\begin{align}
    h(W|Y_{i}, y^{i-1}) &= \frac12 h_m^{i-1} + \frac12 h_M^{i-1}
    \notag
    \\ &\ge \frac12(h(W|y^{i-1}) - \gamma(d)) + \frac12 h_M^{i-1}
    \notag
    \\ &\ge \frac12(h(W|y^{i-1}) - \gamma(d)) + \frac12 h(W|y^{i})
    \notag
\end{align}
where the first inequality follows from \eqref{eq:martingale_deviation_lower_equi} and the second inequality follows from the definition of $h_M$. From the non-negativity of mutual information, we have that $h(W|Y_{i}, y^{i-1}) \leq h(W|y^{i-1})$, implying
\begin{align}
    h(W|y^{i-1}) &\ge \frac12(h(W|y^{i-1}) - \gamma(d)) + \frac12 h(W|y^{i})
    \notag
    \\ h(W|y^{i-1}) &- h(W|y^{i}) \ge - \gamma(d) \label{eq:other_dir}
\end{align}
Combining \eqref{eq:other_dir} with \eqref{eq:martingale_deviation_lower_equi} we have the desired result.
\end{proof}

\clearpage

\subsection{Additional experiments}
\paragraph{Performance across dimensions:} Figure~\ref{fig:dim-exp} plots $\MSE$ against embedding dimension averaged across all trials at both 20 and 60 queries asked. For all dimensions across all experiments, the learned Yummly \texttt{Food-10k} embedding was centered and scaled by a constant amount such that the unit hypercube of user preference points would be contained in the embedding of items, allowing for a rich pool of pairs to be selected from for any user point. This scaling constant was heuristically set to $\sqrt{d}/(3\tilde{\lambda}^{1/2})$, where $\tilde{\lambda}$ is the smallest eigenvalue of the covariance matrix of embedding items. This scaling is motivated by setting the smallest variance direction of the embedding to align with the furthest point of the unit cube at a distance of $\sqrt{d}$ from the origin. For each learned embedding, responses to the Yummly \texttt{Food-10k} training triplets were predicted by selecting the closer of the two comparison items to the reference item, using the embedding to measure distances. For a given embedding, we refer to the fraction of incorrectly predicted triplet responses as the \emph{triplet error fraction}, which we plot for reference against embedding dimension in Figure~\ref{fig:trip_err}.
\begin{figure}[htb]
    \def\fr{0.33}
    \def\fh{1.2}
    \centering
    \begin{subfigure}[t]{\fr\textwidth}
        \centering
        \includegraphics[height=\fh in]{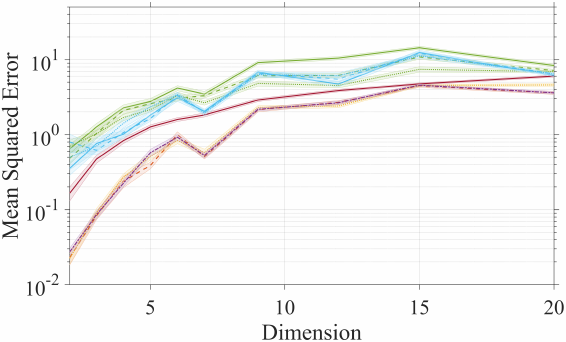}
        \caption{Matched ``constant'' noise: 20 queries}
        \label{fig:dim-con20}
    \end{subfigure}%
    \begin{subfigure}[t]{\fr\textwidth}
        \centering
        \includegraphics[height=\fh in]{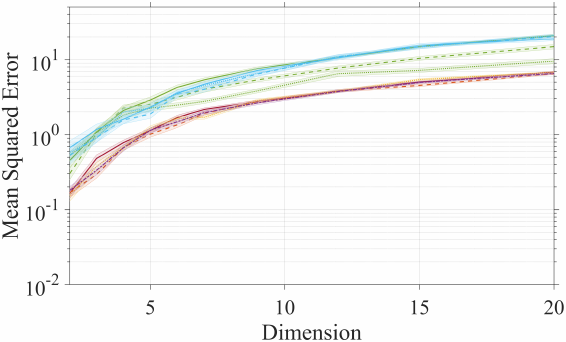}
        \caption{Matched ``normalized'' noise: 20 queries}
        \label{fig:dim-norm20}
    \end{subfigure}%
    \begin{subfigure}[t]{\fr\textwidth}
        \centering
        \includegraphics[height=\fh in]{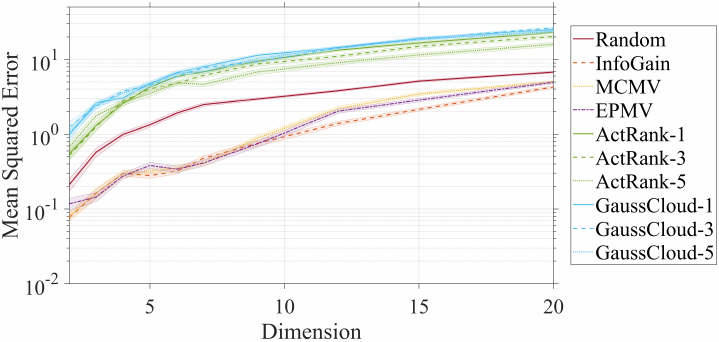}
        \caption{Matched ``decaying'' noise: 20 queries}
        \label{fig:dim-dec20}
    \end{subfigure}\\
    \begin{subfigure}[t]{\fr\textwidth}
        \centering
        \includegraphics[height=\fh in]{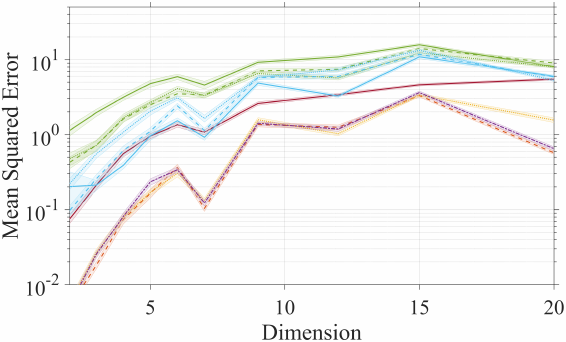}
        \caption{Matched ``constant'' noise: 60 queries}
        \label{fig:dim-con60}
    \end{subfigure}%
    \begin{subfigure}[t]{\fr\textwidth}
        \centering
        \includegraphics[height=\fh in]{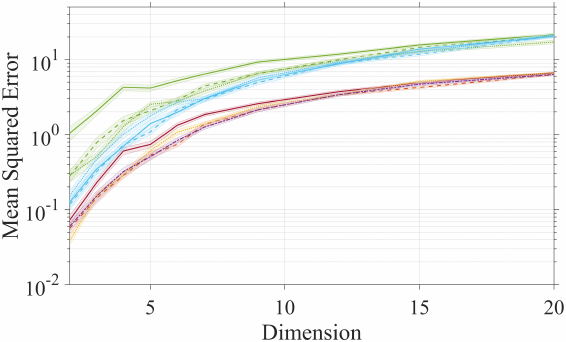}
        \caption{Matched ``normalized'' noise: 60 queries}
        \label{fig:dim-norm60}
    \end{subfigure}%
    \begin{subfigure}[t]{\fr\textwidth}
        \centering
        \includegraphics[height=\fh in]{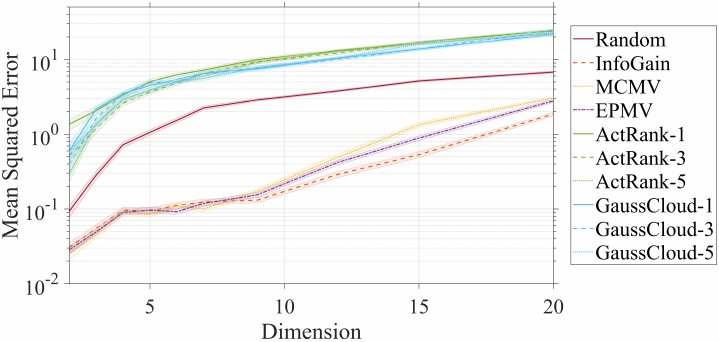}
        \caption{Matched ``decaying'' noise: 60 queries}
        \label{fig:dim-dec60}
    \end{subfigure}
    \caption{Mean squared error performance across dimensions at a fixed number of answered queries, plotted with $\pm$ one standard error.}
    \label{fig:dim-exp}
\end{figure}
\begin{figure}[htb]
    \centering
    \includegraphics[width=0.5\textwidth]{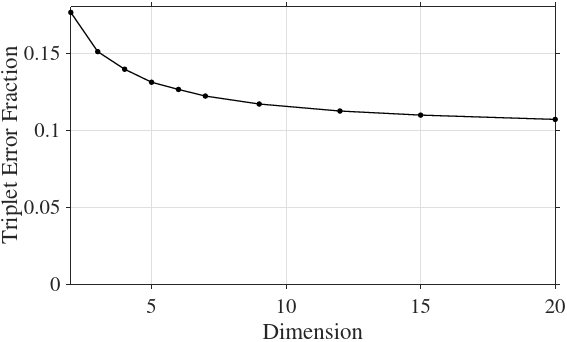}
    \caption{Triplet error fraction versus embedding dimension.}
    \label{fig:trip_err}
\end{figure}

\clearpage

\paragraph{Speed plot comparison}
Figure~\ref{fig:speed-plot} plots $\MSE$ against cumulative compute time for matched logistic noise with ``normalized'' noise constant for $d\in\{4,7,12\}$ in a smaller scale experiment of 60 queries per trial, and 40 trials per dimension. Specifically, $\MSE$ and average cumulative compute time were calculated for \emph{each} number of queries asked, and these two values plotted against each other directly in a range up to 600 seconds. We evaluated all three of our methods (InfoGain, MCMV, EPMV) at various pair pool downsampling rates of $\beta \in \{10^{-3},10^{-3.5},10^{-4}\}$, as listed in the figure legend next to each method. Each experiment was run on an Intel Xeon CPU E5-2680 v4 2.40 GHz processor. 

\begin{figure}[htb]
    \def\fr{0.33}
    \def\fh{1.2}
    \centering
    \begin{subfigure}[t]{\fr\textwidth}
        \centering
        \includegraphics[height=\fh in]{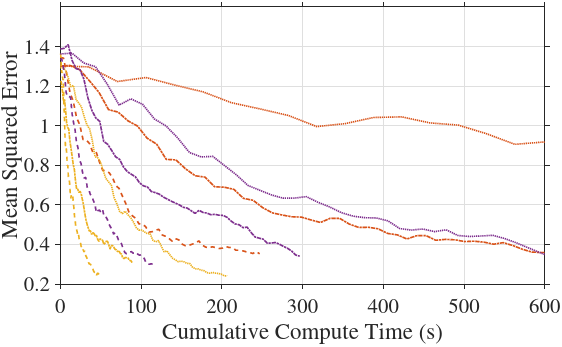}
        \caption{$d$ = 4}
        \label{fig:speed-plot-4}
    \end{subfigure}%
    \begin{subfigure}[t]{\fr\textwidth}
        \centering
        \includegraphics[height=\fh in]{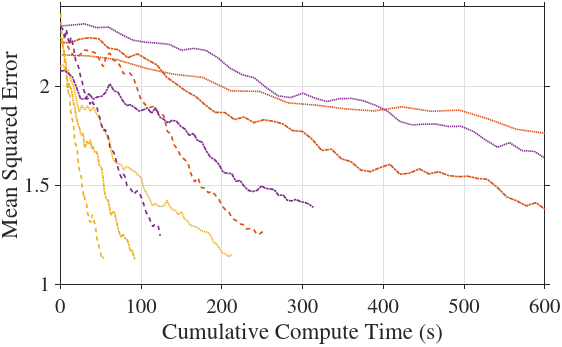}
        \caption{$d$ = 7}
        \label{fig:speed-plot-7}
    \end{subfigure}%
    \begin{subfigure}[t]{\fr\textwidth}
        \centering
        \includegraphics[height=\fh in]{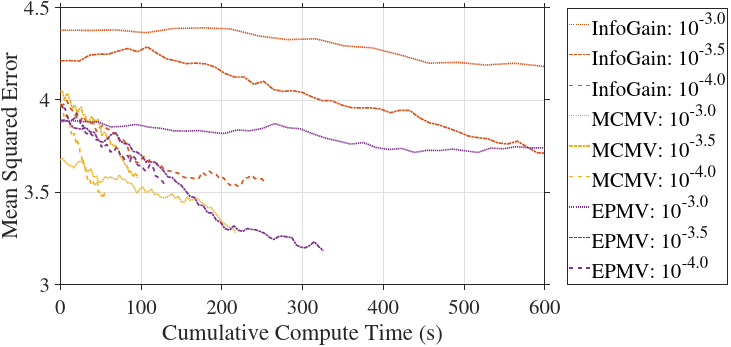}
        \caption{$d$ = 12}
        \label{fig:speed-plot-12}
    \end{subfigure}
    \caption{Mean squared error performance against cumulative compute time (s) for matched, ``normalized'' logistic noise at various pair downsampling rates. Error bars have been omitted for visual clarity.}
    \label{fig:speed-plot}
\end{figure}

\paragraph{Additional experimental results}
In this section, $\MSE$ is evaluated for both matched and mismatched noise at $d\in\{3,5,7,9,12\}$ in Figs.~\ref{fig:mse_3}, \ref{fig:mse_5}, \ref{fig:mse_7}, \ref{fig:mse_9}, and \ref{fig:mse_12}. The model for $k_{pq}$ on mismatched Gaussian noise is chosen as the maximum-likelihood model (``constant,'' ``normalized'', ``decaying'') on the training triplets, calculated separately for each embedding dimension. For all experiments, $\beta=10^{-3}$ and results are averaged over 50 trials.
\begin{figure}[htb]
    \def\fr{0.33}
    \def\fh{1.2}
    \centering
    \begin{subfigure}[t]{\fr\textwidth}
        \centering
        \includegraphics[height=\fh in]{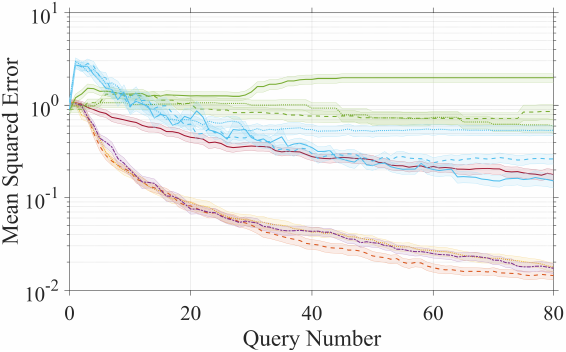}
        \caption{``Constant'' model,\\matched, $d$ = 3}
        \label{fig:mse_3_1}
    \end{subfigure}%
    \begin{subfigure}[t]{\fr\textwidth}
        \centering
        \includegraphics[height=\fh in]{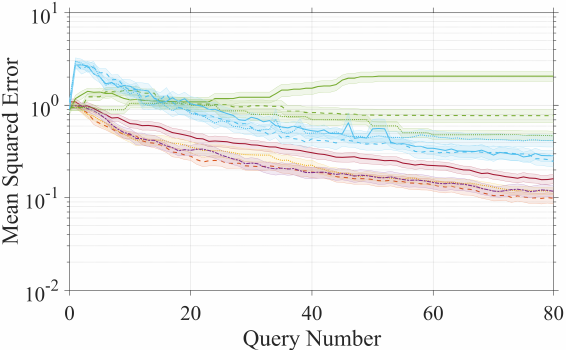}
        \caption{``Normalized'' model,\\matched, $d$ = 3}
        \label{fig:mse_3_2}
    \end{subfigure}%
    \begin{subfigure}[t]{\fr\textwidth}
        \centering
        \includegraphics[height=\fh in]{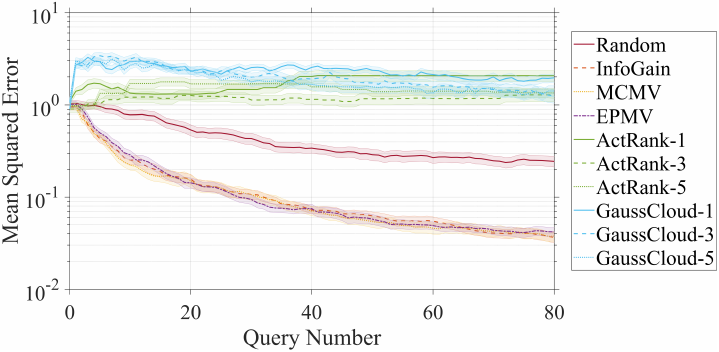}
        \caption{``Decaying'' model,\\matched, $d$ = 3}
        \label{fig:dim-mse_3_3}
    \end{subfigure}\\
    \begin{subfigure}[t]{\fr\textwidth}
        \centering
        \includegraphics[height=\fh in]{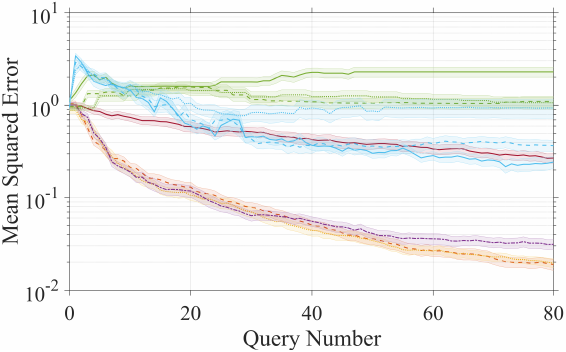}
        \caption{``Constant'' model,\\mismatched, $d$ = 3}
        \label{fig:dim-mse_3_4}
    \end{subfigure}%
    \begin{subfigure}[t]{\fr\textwidth}
        \centering
        \includegraphics[height=\fh in]{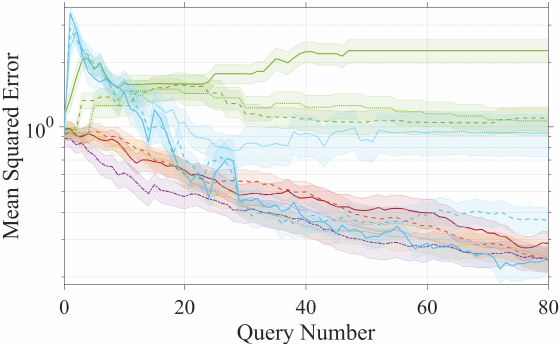}
        \caption{``Normalized'' model,\\mismatched, $d$ = 3}
        \label{fig:mse_3_5}
    \end{subfigure}%
    \begin{subfigure}[t]{\fr\textwidth}
        \centering
        \includegraphics[height=\fh in]{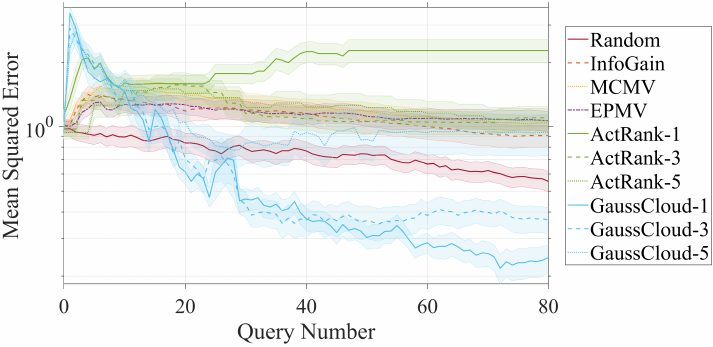}
        \caption{``Decaying'' model,\\mismatched, $d$ = 3}
        \label{fig:mse_3_6}
    \end{subfigure}
    \caption{Mean squared error performance versus number of queries asked for pairwise search in 3 dimensions, plotted with $\pm$ one standard error. All mismatched noise is Gaussian with a ``constant'' noise constant.}
    \label{fig:mse_3}
\end{figure}
\begin{figure}[htb]
    \def\fr{0.33}
    \def\fh{1.2}
    \centering
    \begin{subfigure}[t]{\fr\textwidth}
        \centering
        \includegraphics[height=\fh in]{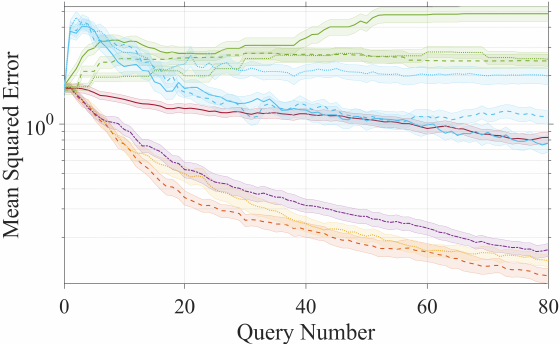}
        \caption{``Constant'' model,\\matched, $d$ = 5}
        \label{fig:mse_5_1}
    \end{subfigure}%
    \begin{subfigure}[t]{\fr\textwidth}
        \centering
        \includegraphics[height=\fh in]{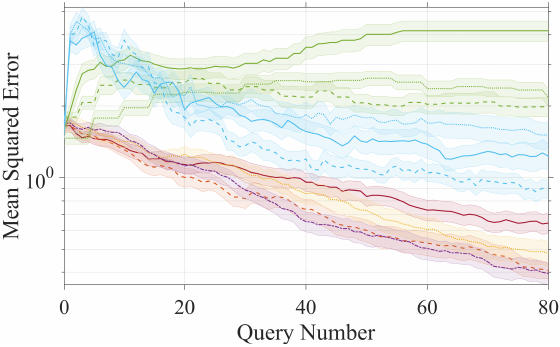}
        \caption{``Normalized'' model,\\matched, $d$ = 5}
        \label{fig:mse_5_2}
    \end{subfigure}%
    \begin{subfigure}[t]{\fr\textwidth}
        \centering
        \includegraphics[height=\fh in]{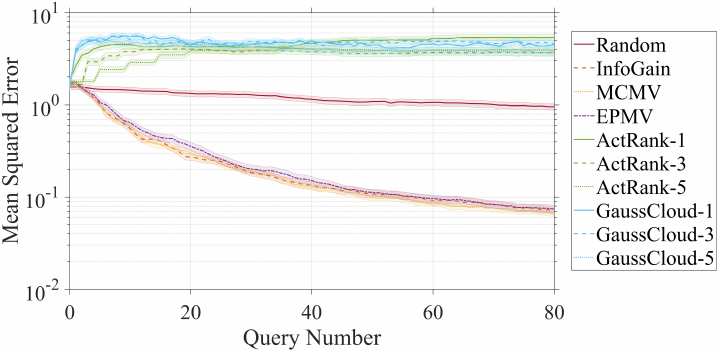}
        \caption{``Decaying'' model,\\matched, $d$ = 5}
        \label{fig:dim-mse_5_3}
    \end{subfigure}\\
    \begin{subfigure}[t]{\fr\textwidth}
        \centering
        \includegraphics[height=\fh in]{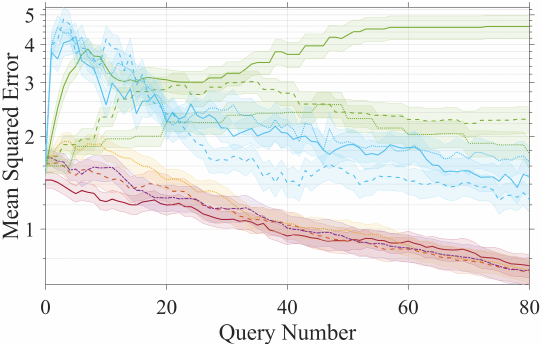}
        \caption{``Constant'' model,\\mismatched, $d$ = 5}
        \label{fig:dim-mse_5_4}
    \end{subfigure}%
    \begin{subfigure}[t]{\fr\textwidth}
        \centering
        \includegraphics[height=\fh in]{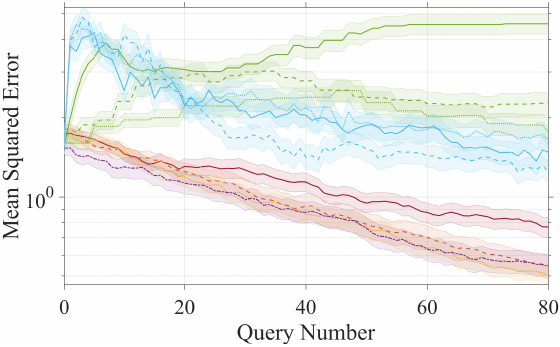}
        \caption{``Normalized'' model,\\mismatched, $d$ = 5}
        \label{fig:mse_5_5}
    \end{subfigure}%
    \begin{subfigure}[t]{\fr\textwidth}
        \centering
        \includegraphics[height=\fh in]{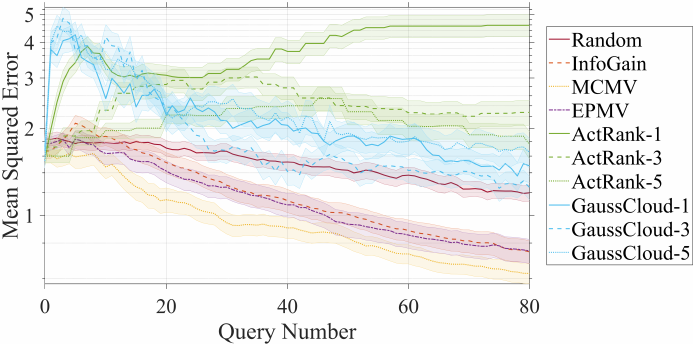}
        \caption{``Decaying'' model,\\mismatched, $d$ = 5}
        \label{fig:mse_5_6}
    \end{subfigure}
    \caption{Mean squared error performance versus number of queries asked for pairwise search in 5 dimensions, plotted with $\pm$ one standard error. All mismatched noise is Gaussian with a ``normalized'' noise constant.}
    \label{fig:mse_5}
\end{figure}
\begin{figure}[htb]
    \def\fr{0.33}
    \def\fh{1.2}
    \centering
    \begin{subfigure}[t]{\fr\textwidth}
        \centering
        \includegraphics[height=\fh in]{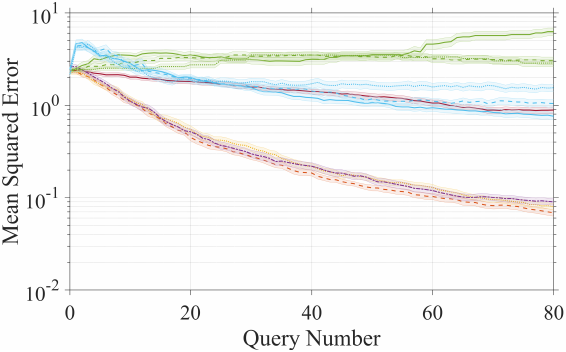}
        \caption{``Constant'' model,\\matched, $d$ = 7}
        \label{fig:mse_7_1}
    \end{subfigure}%
    \begin{subfigure}[t]{\fr\textwidth}
        \centering
        \includegraphics[height=\fh in]{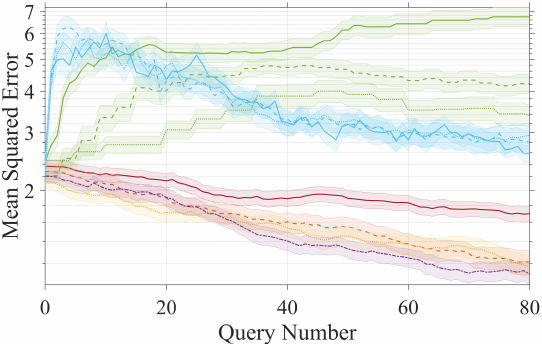}
        \caption{``Normalized'' model,\\matched, $d$ = 7}
        \label{fig:mse_7_2}
    \end{subfigure}%
    \begin{subfigure}[t]{\fr\textwidth}
        \centering
        \includegraphics[height=\fh in]{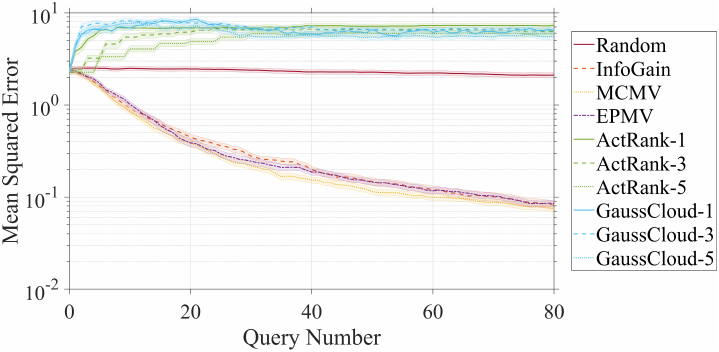}
        \caption{``Decaying'' model,\\matched, $d$ = 7}
        \label{fig:dim-mse_7_3}
    \end{subfigure}\\
    \begin{subfigure}[t]{\fr\textwidth}
        \centering
        \includegraphics[height=\fh in]{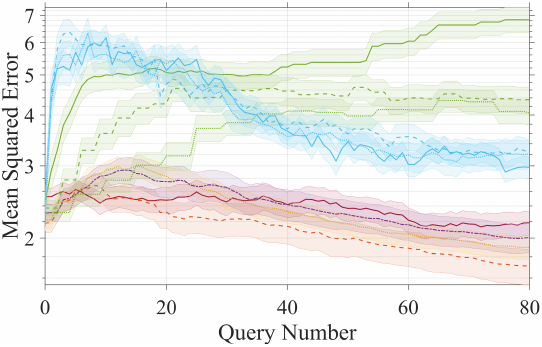}
        \caption{``Constant'' model,\\mismatched, $d$ = 7}
        \label{fig:dim-mse_7_4}
    \end{subfigure}%
    \begin{subfigure}[t]{\fr\textwidth}
        \centering
        \includegraphics[height=\fh in]{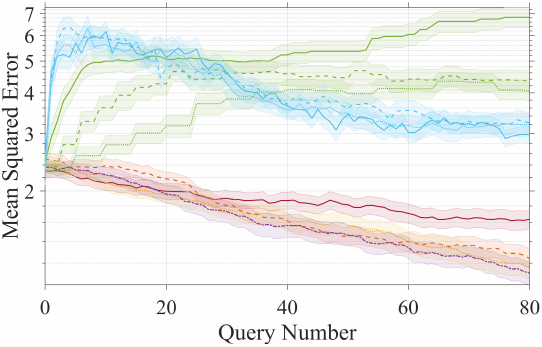}
        \caption{``Normalized'' model,\\mismatched, $d$ = 7}
        \label{fig:mse_7_5}
    \end{subfigure}%
    \begin{subfigure}[t]{\fr\textwidth}
        \centering
        \includegraphics[height=\fh in]{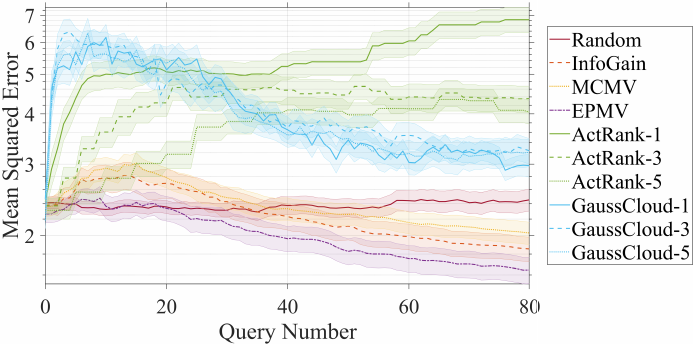}
        \caption{``Decaying'' model,\\mismatched, $d$ = 7}
        \label{fig:mse_7_6}
    \end{subfigure}
    \caption{Mean squared error performance versus number of queries asked for pairwise search in 7 dimensions, plotted with $\pm$ one standard error. All mismatched noise is Gaussian with a ``normalized'' noise constant.}
    \label{fig:mse_7}
\end{figure}
\begin{figure}[htb]
    \def\fr{0.33}
    \def\fh{1.2}
    \centering
    \begin{subfigure}[t]{\fr\textwidth}
        \centering
        \includegraphics[height=\fh in]{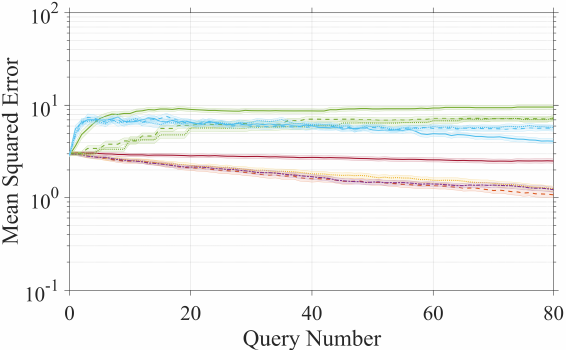}
        \caption{``Constant'' model,\\matched, $d$ = 9}
        \label{fig:mse_9_1}
    \end{subfigure}%
    \begin{subfigure}[t]{\fr\textwidth}
        \centering
        \includegraphics[height=\fh in]{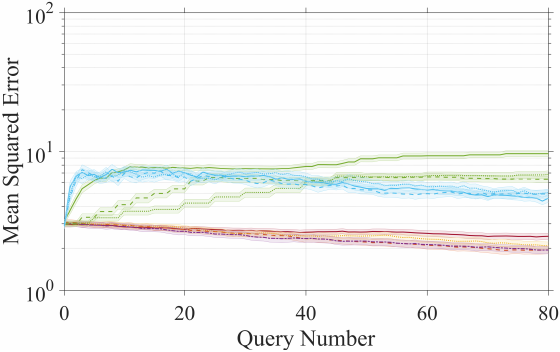}
        \caption{``Normalized'' model,\\matched, $d$ = 9}
        \label{fig:mse_9_2}
    \end{subfigure}%
    \begin{subfigure}[t]{\fr\textwidth}
        \centering
        \includegraphics[height=\fh in]{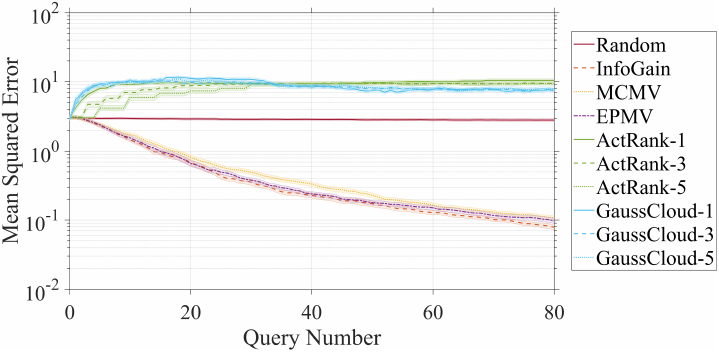}
        \caption{``Decaying'' model,\\matched, $d$ = 9}
        \label{fig:dim-mse_9_3}
    \end{subfigure}\\
    \begin{subfigure}[t]{\fr\textwidth}
        \centering
        \includegraphics[height=\fh in]{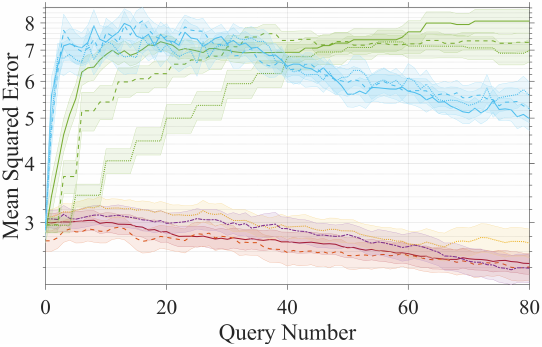}
        \caption{``Constant'' model,\\mismatched, $d$ = 9}
        \label{fig:dim-mse_9_4}
    \end{subfigure}%
    \begin{subfigure}[t]{\fr\textwidth}
        \centering
        \includegraphics[height=\fh in]{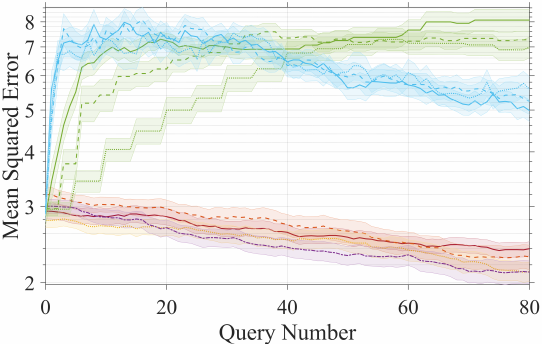}
        \caption{``Normalized'' model,\\mismatched, $d$ = 9}
        \label{fig:mse_9_5}
    \end{subfigure}%
    \begin{subfigure}[t]{\fr\textwidth}
        \centering
        \includegraphics[height=\fh in]{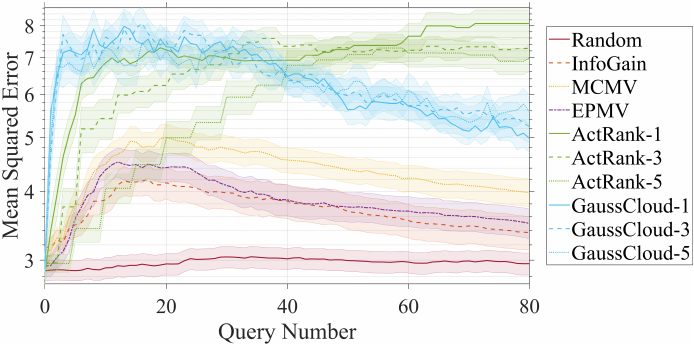}
        \caption{``Decaying'' model,\\mismatched, $d$ = 9}
        \label{fig:mse_9_6}
    \end{subfigure}
    \caption{Mean squared error performance versus number of queries asked for pairwise search in 9 dimensions, plotted with $\pm$ one standard error. All mismatched noise is Gaussian with a ``normalized'' noise constant.}
    \label{fig:mse_9}
\end{figure}
\begin{figure}[htb]
    \def\fr{0.33}
    \def\fh{1.2}
    \centering
    \begin{subfigure}[t]{\fr\textwidth}
        \centering
        \includegraphics[height=\fh in]{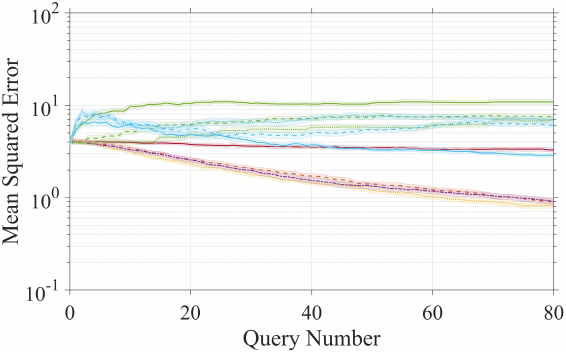}
        \caption{``Constant'' model,\\matched, $d$ = 12}
        \label{fig:mse_12_1}
    \end{subfigure}%
    \begin{subfigure}[t]{\fr\textwidth}
        \centering
        \includegraphics[height=\fh in]{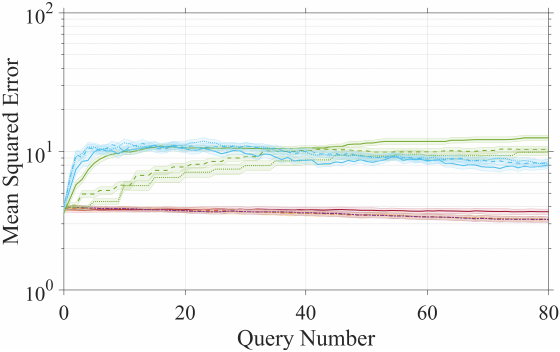}
        \caption{``Normalized'' model,\\matched, $d$ = 12}
        \label{fig:mse_12_2}
    \end{subfigure}%
    \begin{subfigure}[t]{\fr\textwidth}
        \centering
        \includegraphics[height=\fh in]{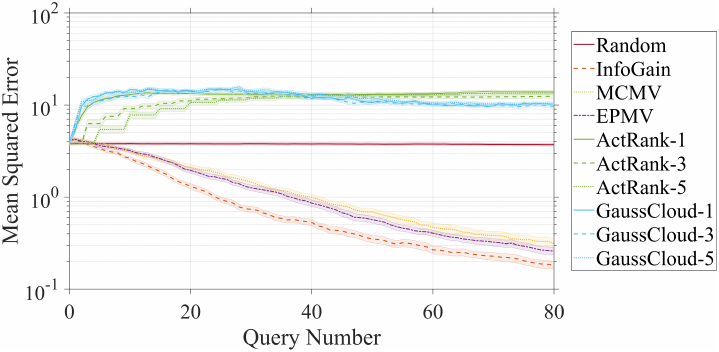}
        \caption{``Decaying'' model,\\matched, $d$ = 12}
        \label{fig:dim-mse_12_3}
    \end{subfigure}\\
    \begin{subfigure}[t]{\fr\textwidth}
        \centering
        \includegraphics[height=\fh in]{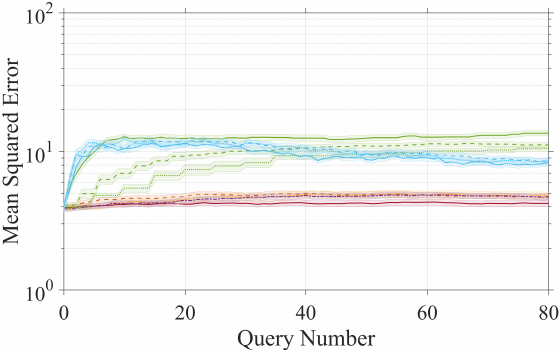}
        \caption{``Constant'' model,\\mismatched, $d$ = 12}
        \label{fig:dim-mse_12_4}
    \end{subfigure}%
    \begin{subfigure}[t]{\fr\textwidth}
        \centering
        \includegraphics[height=\fh in]{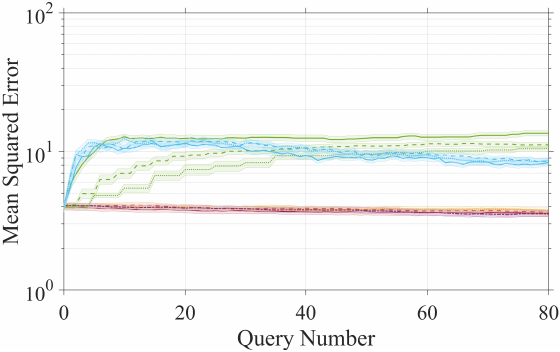}
        \caption{``Normalized'' model,\\mismatched, $d$ = 12}
        \label{fig:mse_12_5}
    \end{subfigure}%
    \begin{subfigure}[t]{\fr\textwidth}
        \centering
        \includegraphics[height=\fh in]{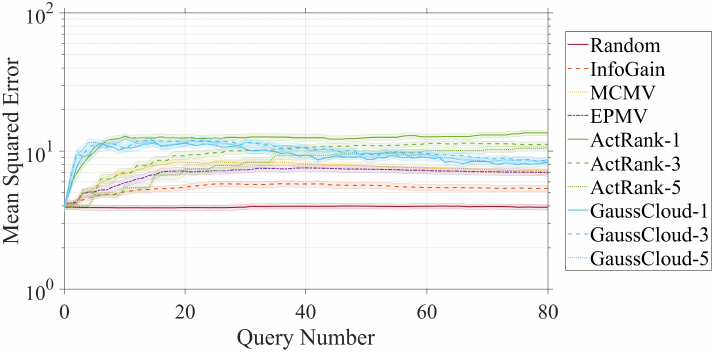}
        \caption{``Decaying'' model,\\mismatched, $d$ = 12}
        \label{fig:mse_12_6}
    \end{subfigure}
    \caption{Mean squared error performance versus number of queries asked for pairwise search in 12 dimensions, plotted with $\pm$ one standard error. All mismatched noise is Gaussian with a ``normalized'' noise constant.}
    \label{fig:mse_12}
\end{figure}

\end{document}